\documentclass[11pt]{article}

\usepackage[final]{acl}

\usepackage{times}
\usepackage{latexsym}

\usepackage[T1]{fontenc}

\usepackage[utf8]{inputenc}

\usepackage{microtype}

\usepackage{inconsolata}

\usepackage{graphicx}

\usepackage{xcolor}
\usepackage{tcolorbox}
\definecolor{srapromptpink}{HTML}{E67897}
\definecolor{srapromptpinkbg}{HTML}{FFE4EC}
\definecolor{srapromptblue}{HTML}{5FCBE6}
\definecolor{srapromptbluebg}{HTML}{DEF8FF}
\tcbset{srapromptstyle/.style={
  coltitle=white,fonttitle=\bfseries,fontupper=\small,
  boxrule=0.6pt,arc=1.5mm,
  left=1.5mm,right=1.5mm,top=1.2mm,bottom=1.2mm,
  before skip=7pt,after skip=9pt
}}
\newtcolorbox{sratrainingprompt}{
  srapromptstyle,
  title=OPD training prompt,
  colframe=srapromptpink,colback=srapromptpinkbg,
  colbacktitle=srapromptpink
}
\newtcolorbox{sraevaluationprompt}{
  srapromptstyle,
  title=Evaluation prompt,
  colframe=srapromptblue,colback=srapromptbluebg,
  colbacktitle=srapromptblue
}

\usepackage{enumitem}
\usepackage{amsmath,amssymb,amsthm}
\newtheorem{theorem}{Theorem}
\newtheorem*{theorem*}{Theorem}
\newtheorem{lemma}{Lemma}
\newtheorem{remark}{Remark}

\usepackage{algorithm}
\usepackage{algpseudocode}
\algrenewcommand\algorithmicrequire{\textbf{Input:}}
\algrenewcommand\algorithmicensure{\textbf{Output:}}

\usepackage{booktabs}
\usepackage{placeins}
\usepackage{array,colortbl,multirow}

\definecolor{gain}{RGB}{0,125,65}
\definecolor{drop}{RGB}{190,45,45}

\tcbuselibrary{skins}
\definecolor{sragk}{HTML}{D6E8FA}       
\definecolor{sragc}{HTML}{FFCF9E}       
\definecolor{sraexframe}{HTML}{4A5D73}
\definecolor{sraexback}{HTML}{FBFCFD}
\newcommand{\gk}[1]{{\setlength{\fboxsep}{0pt}\colorbox{sragk}{\strut #1}}}
\newcommand{\gc}[1]{{\setlength{\fboxsep}{0pt}\colorbox{sragc}{\strut #1}}}
\newcommand{\sraexrule}{\par\smallskip{\color{sraexframe!30}\hrule height 0.4pt}\smallskip}
\newcounter{sraexample}
\newtcolorbox{sraexbox}[1]{
  enhanced,
  colframe=sraexframe,colback=sraexback,coltitle=white,colbacktitle=sraexframe,
  fonttitle=\bfseries\small,fontupper=\small,
  boxrule=0.6pt,arc=1.5mm,
  left=3mm,right=3mm,top=2mm,bottom=2mm,
  before skip=0pt,after skip=0pt,
  title={#1}
}

\title{SRA: Selective Residual Adaptation for LLM Reasoning Enhancement}

\author{
  Zhenxing Dou$^{1}$, Kun He$^{2}$, Mingyang Yi$^{2}$\thanks{Corresponding authors: Mingyang Yi and Yong Liu.}, 
  Yong Liu$^{1}$\footnotemark[1] \\
  $^{1}$Gaoling School of Artificial Intelligence, $^{2}$School of Information \\
  Renmin University of China, Beijing, China \\
  \texttt{\{yimingyang,liuyonggsai\}@ruc.edu.cn}
}

\begin{document}
\maketitle

\begin{abstract}

Knowledge distillation transfers reasoning capabilities from large language models to lightweight students. However, distilled small models still struggle with multi-step reasoning. We observe that the distribution divergences between teacher and student models along reasoning paths are highly sparse, and target tokens almost always fall within the student top-$K$ candidate set. Based on this finding, unlike approaches updating student backbones or requiring runtime teacher computation, we propose Selective Residual Adaptation (SRA), a plug-and-play framework that internalizes localized corrections. Concretely, SRA employs a Discrepancy Gate to identify divergence positions alongside a Candidate Residual Adapter for top-$K$ logit calibration to correct the mismatch. Theoretically, we establish error bounds ensuring strict sequence level performance gains. Across mathematical benchmarks on Qwen3 models from 0.6B to 4B, SRA consistently enhances distilled student models, raising average Pass@16 by up to 4.37 percentage points while preserving general capabilities.

\end{abstract}

\section{Introduction}
\label{sec:introduction}

Large language Models (LLMs) \citep{yang2025qwen3} have demonstrated strong performance on complex multi-step reasoning tasks \citep{wei2022chain} such as math \citep{shao2024deepseekmath}, coding \citep{chen2021evaluating}, and data analysis \citep{hong2025datainterpreter}. However, the substantial computational costs of LLMs and GPU memory requirements hinder efficient deployment \citep{kwon2023efficient}. To resolve this, researchers propose to use smaller language models distilled \citep{hinton2015distilling} from the larger ones to offer lower computational requirements. However, distilled students still struggle with complex reasoning \citep{shridhar2023distilling}, as their limited capacity and the distribution mismatch constrain how well they learn from the teacher, leaving a performance gap even after distillation \citep{agarwal2024onpolicy,li2025small}.

Due to this mismatch, collaborative inference offers an alternative approach to improving student reasoning \citep{huang2026relayllm, ong2025routellm,li2026ets}. In this regime, the student handles most generation steps and requests teacher corrections only at selected positions \citep{fu2025r2r,shen2026reasoning}. However, these interventions still require teacher computation at runtime, and synchronization and prefix processing can introduce additional latency and complicate deployment \citep{ye2026select}. These costs can erode the efficiency advantages of running the student alone. This raises a central question: \textit{Can we internalize these local corrections in a lightweight mechanism that operates independently, further improving reasoning without updating the student backbone or querying the teacher during inference?}

To investigate this possibility, we first compare the next-token distributions of teacher and student at shared prefixes. Our diagnostics reveal two complementary observations. First, positions where the student diverges from the teacher's preferred token are highly sparse and can be sensitively detected by total variation (TV) distance. Second, the teacher's preferred token is almost always included in the student's top-$K$ candidates. Together, these observations motivate learning local corrections by reallocating probability among existing student candidates at selected positions.

Inspired by these findings, we propose \textbf{Selective Residual Adaptation (SRA)}, a lightweight, plug-and-play framework that learns output-side probability corrections while keeping the student backbone frozen. Our SRA consists of two critical components. A light \textbf{Discrepancy Gate (DG)} predicts high-disagreement positions from hidden states of the frozen student, using supervision derived from TV distance between small and large models. A \textbf{Candidate Residual Adapter (CRA)} learns a lightweight additive correction module to correct the student's candidate logits by matching the teacher's conditional distribution over the same top-$K$ candidates. During generation, CRA is applied only at highly mismatched positions selected by DG. At all other positions, the original student distribution is retained. Both lightweight modules reuse student features and operate without teacher queries at inference, thereby enabling efficient corrections with minimal computational overhead.

Furthermore, we conduct a theoretical analysis of SRA, which establishes sufficient conditions for strict improvement in sequence-level success. Experiments on four mathematical reasoning benchmarks with Qwen3-0.6B, 1.7B, and 4B show that SRA further improves the macro-averaged accuracy of students trained with different distillation methods including standard knowledge distillation~\citep{hinton2015distilling,kim2016sequence} and on-policy distillation \citep{lu2025onpolicydistillation}. Our contributions are summarized as follows:
\begin{itemize}[noitemsep, topsep=3pt]
    \item \textbf{Selective residual adaptation.} We propose SRA, a plug-and-play framework that internalizes localized corrections without student updates or inference-time teacher queries.
    \item \textbf{Theoretical improvement guarantees.} We derive sufficient error conditions ensuring that SRA achieves strict sequence level performance gains over the base student.
    \item \textbf{Empirical gains and capability preservation.} We demonstrate consistent gains across various models and distillation methods, while preserving general capabilities.
\end{itemize}

\section{Related Work}
\label{sec:related_work}

\textbf{Knowledge Distillation.} To close the gap between large and small models, knowledge distillation proposes to align student with teacher w.r.t. soft predictions \citep{butt2026soft}, generated sequences \citep{kim2016sequence}, or reasoning explanations \citep{hsieh2023distilling,jiang2026ofa}. Recently, the on-policy distillation (OPD) \citep{lu2025onpolicydistillation} was proposed to improve knowledge distillation by obtaining teacher supervision at student-generated prefixes to mitigate the distribution shift between training and inference \citep{gu2024minillm, agarwal2024onpolicy,wang2025improved}. Building upon OPD, subsequent works improve it by refining distribution matching objectives and selecting token supervision \citep{jin2026entropy,kim2026explain}. Nevertheless, these distillation methods cannot eliminate the gap between teacher and student, so that our ``predict then correct'' SRA can be used to boost these methods as we did in Section \ref{sec:experiments}. 

\textbf{Collaborative Inference.} Collaboration between small and large models balances reasoning quality against computational cost through request routing or model switching during generation \citep{ong2025routellm,fu2025r2r}. The request routing selects a model for the entire input, whereas collaboration during generation lets the smaller model generate most tokens and seek guidance from the larger model at selected positions. Learned invocation policies or predictive disagreement signals determine where the larger model takes over \citep{huang2026relayllm,shen2026reasoning}. Although these direct interventions reduce the amount of generation handled by the larger model, they still require its predictions at runtime. In contrast, SRA replaces the teacher at inference with a lightweight extra module, as in EAGLE~\citep{li2025eagle3}.

\textbf{Parameter-Efficient Fine-Tuning.} Parameter-efficient fine-tuning \citep{hu2022lora,xin2024v} reduces the number of trainable parameters through additional modules \citep{houlsby2019parameter}, low-rank updates \citep{hu2022lora,li2026reasoning}, or continuous prefixes \citep{li2021prefix}, and can incorporate quantization to lower training memory requirements \citep{houlsby2019parameter,hu2022lora,li2021prefix,dettmers2023qlora}. Alternatively, prior work trains an independent side network conditioned on intermediate features from a frozen backbone to bypass backpropagation \citep{zhang2020side,sung2022lst,su2025reveal,zhang2026reward}. Although SRA also leaves the backbone untouched, it introduces two lightweight modules atop the frozen student that intervene only at selected token positions, substantially reducing both training and inference overhead.

\begin{figure*}[ht]
\centering
\includegraphics[width=\textwidth]{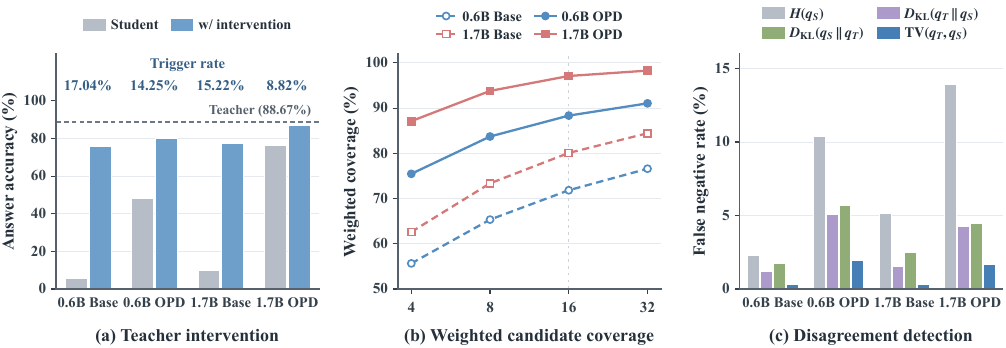}
\caption{ (a) Answer accuracy with and without teacher sampling triggered by top-1 disagreement. The annotations show trigger rates, and the dashed line indicates independent teacher accuracy. (b) The fraction of the teacher's top-$K$ probability mass covered by the student's top-$K$ set at positions where their top-1 predictions differ. (c) The fraction of top-1 disagreements within the student’s top-$K$ candidate set ($K=16$) missed by TV, forward KL, reverse KL, and student entropy when each measure selects the same number of positions. We select 256 questions from the GSM8K test split~\citep{cobbe2021gsm8k} and compare Qwen3-0.6B-Base and 1.7B-Base and OPD students~\citep{yang2025qwen3} with Qwen3-8B as the teacher. Appendix~\ref{app:empirical-motivation} provides training and evaluation details.}
\label{fig:tv-motivation}
\vspace{-3mm}
\end{figure*}

\section{Preliminary and Empirical Motivation}
\label{sec: Preliminary}
\subsection{Problem Formulation}
\label{sec:problem}
In this section, we introduce the necessary preliminaries in this paper. Considering an input question $x \sim \mathcal{D}$, the student model $\pi_{S}$ autoregressively generates a response $y = (y_1, \ldots, y_H)$ of length $H$. Here, $y_{i}$ is a single token and the student model can be a pretrained small model, a previously distilled model, or a post-trained model operating in non-thinking mode in this paper. The output probability of student model can be factorized as 
\begin{equation}
    \pi_S(y\mid x) =\prod_{t=1}^{H}\pi_S(y_{t}\mid x, y_{< t}).
\end{equation}
During distillation, our goal is to align the output of teacher and student model, where teacher model $\pi_T(y\mid x)$ is a policy model with superior reasoning capabilities, such as a larger-scale model or a model operating in thinking mode. The distillation process is usually conducted by minimizing the forward KL divergence $D_{\mathrm{KL}}(\pi_{T}\parallel \pi_{S})$ i.e., knowledge distillation \citep{kim2016sequence} or reverse KL divergence $D_{KL}(\pi_{S}\parallel \pi_{T})$, i.e., on-policy distillation \citep{lu2025onpolicydistillation}.

\subsection{Empirical Motivation}
\label{sec:motivation}   

While distillation aims to align student outputs with the teacher, some mismatch remains inevitable. We therefore seek an indicator to determine which positions to correct during student inference.

\paragraph{Mismatch on Top-1 Token.} As empirically observed in \citet{tsur2025heavywater}, the top-1 token usually dominates the output probability of LLMs. We therefore investigate whether fixing positions where student and teacher top-1 tokens disagree can close the gap between them. Specifically, during the inference of student model, we substitute the next token prediction obtained by teacher model whenever their top-1 tokens are different. As can be seen in Figure \ref{fig:tv-motivation}(a), across varied student models, the substitution almost recovers the performance of teacher model even for those unaligned models, raising OPD accuracy from 48.44\% to 80.08\% at 0.6B and from 76.17\% to 87.11\% at 1.7B, close to independent teacher accuracy of 88.67\%. Moreover, the substitutions are only used on a small fraction of tokens ($\approx$ 13.83\% on average), especially for models after distillation. 

To understand why such sparse intervention is so effective, Figure~\ref{fig:tv-motivation}(b) examines the teacher probability mass retained by the student's candidate pool. For a prefix \(s_t=(x,y_{<t})\), let \(\mathcal K_S(s_t)\) and \(\mathcal K_T(s_t)\) denote the student and teacher top-\(K\) sets. We define their weighted intersection coverage as
\begin{equation}
c_K(s_t)=
\frac{
\sum_{y_{t}\in\mathcal K_S(s_t)\cap\mathcal K_T(s_t)}
\pi_T(y_t\mid s_t)
}{
\sum_{y_t\in\mathcal K_T(s_t)}
\pi_T(y_t\mid s_t)
}.
\end{equation}
Evaluating $c_K(s_t)$ exclusively at positions with top-1 disagreements reveals substantial coverage across all settings. This demonstrates that even when top-1 predictions diverge, the student's candidate set retains the vast majority of the teacher's high-probability mass. Hence, mismatches largely stem from re-ranking among shared candidates rather than missing viable tokens, which explains why intervening within the candidate set suffices.

\paragraph{Total Variation (TV) Indicator.} The discussion above indicates that the mismatch between student and teacher top-1 tokens serves as an effective criterion for deciding whether to activate correction. However, we find this indicator is hard to learn in practice (see Appendix~\ref{app:gate-supervision-analysis} for details). Thus, we seek another indicator which is highly related to the top-1 mismatch. Intuitively, the top-1 mismatch results in a gap between the output distributions of student and teacher models. We select token positions exhibiting high total variation distance, forward KL divergence, reverse KL divergence, and entropy \citep{cover2006elements} between the student and teacher to examine whether these distributional metrics can reliably identify mismatched top-1 tokens.

The results are summarized in Figure~\ref{fig:tv-motivation}(c). The output distribution is approximated by the normalized distribution over top-$K$ tokens for both teacher and student models as 
\begin{equation}
q_M(y_{t}\mid s_{t})=
\frac{\pi_M(y_{t}\mid s_{t})}
{\sum_{y_{t}\in\mathcal K(s_{t})}\pi_M(y_{t}\mid s_{t})},
\label{eq:normalized}
\end{equation}
where $\mathcal{K}(s_{t})$ is the student's top-$K$ candidates and $\pi_{M}$ is student or teacher model. We select positions satisfying $\mathrm{TV}(q_T, q_S) \ge 0.1$ and compare against other distributional measures that select an equal number of positions with the highest scores. For these positions, we report their alignments with the positions with mismatched top-1 tokens, represented by false negative rate (FNR):
\begin{equation}
    \mathrm{FNR}=
\frac{\#\{\text{unselected disagreements}\}}
{\#\{\text{all top-1 disagreements}\}}.
\end{equation}
As can be seen, TV achieves the lowest FNR in all four settings, at 0.32\%--1.97\%, indicating that TV can be used as a good indicator in practice. A sensitivity analysis of $\tau$ is provided in Appendix~\ref{app:threshold-sensitivity}.

We further compare how TV and KL respond when the student and teacher distributions are close. Fix a prefix $s_t$ and its student top-$K$ candidate set $\mathcal K(s_t)$. Let $q_T$ denote the teacher distribution normalized on this set. We consider nearby student distributions of the form $q_S^{(\varepsilon)}=q_T+\varepsilon\Delta$, where the fixed vector $\Delta\ne0$ with $\sum_i\Delta_i=0$ specifies changes in candidate probabilities and $\varepsilon>0$ controls their magnitude. 
\begin{theorem}[Local sensitivity of TV and KL]
\label{thm:tv-local-response}
Suppose $q_{T,i}>0$ for every candidate $i$. For sufficiently small $\varepsilon>0$,
\begin{equation}
\operatorname{TV}(q_T, q_S^{(\varepsilon)})
=\frac{\varepsilon}{2}\|\Delta\|_1,
\label{eq:tv-local-expansion}
\end{equation}
\begin{equation}
D_{\mathrm{KL}}(q_T\|q_S^{(\varepsilon)})
=\frac{\varepsilon^2}{2}\sum_{i=1}^K\frac{\Delta_i^2}{q_{T,i}}
+O(\varepsilon^3).
\label{eq:kl-local-expansion}
\end{equation}
The reverse direction $D_{\mathrm{KL}}(q_S^{(\varepsilon)}\|q_T)$ has the same second-order term.
\end{theorem}
TV therefore grows linearly in $\varepsilon$, whereas both KL directions grow only quadratically. A small shift of probability among the candidates, which may already flip the top-1 token, is thus more visible to TV than to KL, consistent with Figure~\ref{fig:tv-motivation}(c). The proof is given in Appendix~\ref{app:tv-local-sensitivity}.

\begin{figure*}[t]
    \centering
    \includegraphics[width=\textwidth]{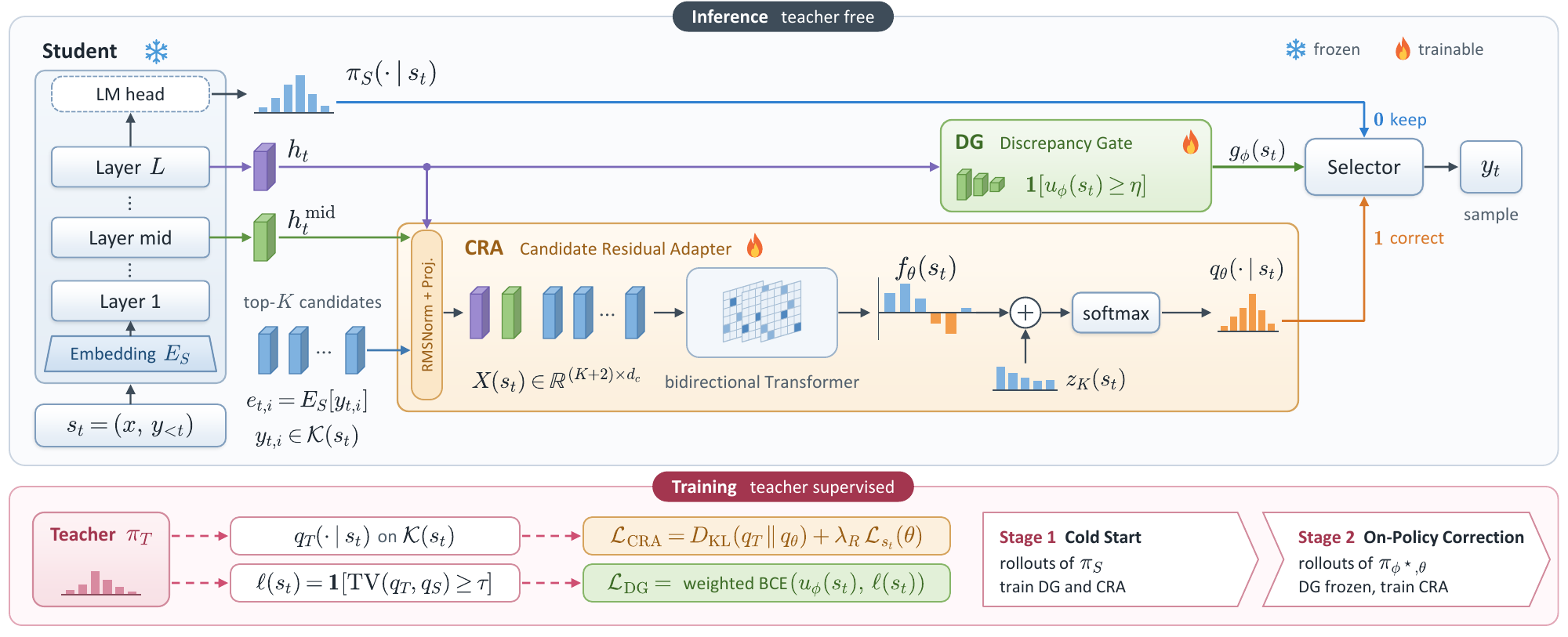}
    \caption{Overview of SRA. The upper panel shows inference. DG reads the final hidden state $h_t$ of the frozen student and decides whether to correct the current step. If $g_\phi(s_t)=1$, CRA predicts logit residuals over the student's top-$K$ candidates and the next token is sampled from $q_\theta$, otherwise it is sampled from $\pi_S$. The lower panel shows training, where the teacher provides $q_T$ and $\ell(s_t)$ to supervise CRA and DG in two stages.}    
    \label{fig:sra-overview}
    \vspace{-3mm}
\end{figure*}

\section{Selective Residual Adaptation}
\label{sec:method}

Building on the empirical findings, we propose \textbf{Selective Residual Adaptation (SRA)} to refine the student's generation distribution through position selection and candidate correction. Intuitively, rectifying the generation requires two essential steps: (i) predicting the positions where intervention is required, and (ii) calibrating the output based on the current hidden state. Thus, our SRA consists of two lightweight modules: a \textbf{Discrepancy Gate (DG)} that identifies positions for correction, and a \textbf{Candidate Residual Adapter (CRA)} that adjusts candidate probabilities at those positions. Figure~\ref{fig:sra-overview} shows how DG controls whether CRA's
candidate correction is applied at each decoding step.

\subsection{Discrepancy Gate}
\label{sec:dg}
\label{sec:teacher_feedback}

The analysis in Section~\ref{sec:motivation} supports using TV distance to indicate whether correction is required. To avoid evaluating the teacher during inference, we train DG to predict intervention positions from student features. The prediction builds upon the final hidden states of the student model. Following \citet{li2024eagle,cai2024medusa,li2025eagle3}, we build DG on this layer, which contains comprehensive representations such that a lightweight module atop the backbone suffices to predict discrepancies between the student and teacher.

Concretely, for each prefix $s_{t}$, DG is a lightweight MLP that takes the student final hidden state $h_{t}$ as input and predicts the discrepancy label
\begin{equation}
\ell(s_{t})=\mathbf 1[\mathrm{TV}(q_{T}(\cdot \mid s_{t}), q_{S}(\cdot\mid s_{t}))\geq\tau].
\label{eq:tv_label}
\end{equation}
The gate score and decision are
\begin{equation}
u_\phi(s_t) = \sigma\!\left(\operatorname{MLP}_\phi(h_t)\right),
\label{eq:gate_prob}
\end{equation}
\begin{equation}
g_\phi(s_t) = \mathbf 1[u_\phi(s_t)\geq\eta], \quad \eta \in (0,1).
\label{eq:gate}
\end{equation}
Here, $\sigma$ is the sigmoid function, which maps the MLP output to a score between 0 and 1. The threshold $\tau$ defines the training label, whereas $\eta$ converts the predicted score into a selection decision.

Since positions with $\ell(s_t)=1$ are much rarer than those with $\ell(s_t)=0$, we train DG with a weighted binary cross-entropy~\citep{Xie_2015_ICCV}:
\begin{equation}
\begin{aligned}
\mathcal L_{\mathrm{DG}}(s_{t})
={}&-w_+\ell(s_{t})\log u_\phi(s_{t})\\
&-(1-\ell(s_{t}))\log(1-u_\phi(s_{t})),
\end{aligned}
\label{eq:gate_loss}
\end{equation}
where $w_+=N_-/N_+$ is the ratio of negative to positive positions in each training batch, which prevents the loss from being dominated by low-discrepancy positions.

\subsection{Candidate Residual Adapter}
\label{sec:cra}
\label{sec:residual_policy}
After specifying the correction positions, we should decide how to fix them. As clarified in Figure \ref{fig:tv-motivation}(b) of Section \ref{sec:motivation}, the teacher's top-$K$ tokens over the full vocabulary are mostly contained in the student's top-$K$ set, even for unaligned small models. We therefore correct the student by re-ranking its top-$K$ tokens when correction is required.

To this end, we propose CRA, which uses a small bidirectional Transformer to predict a logit residual for each candidate. In contrast to DG, its inputs comprise the student's final hidden state $h_t$, an intermediate hidden state $h^{\mathrm{mid}}_t$, and candidate token embeddings, since an intermediate layer retains earlier semantic information~\citep{li2025eagle3}.  

Concretely, let $\mathcal K(s_t)=\{y_{t,1},\ldots,y_{t,K}\}$ denote the student's top-$K$ candidate tokens at prefix $s_t$. For each candidate $y_{t,i}$, we obtain its embedding $e_{t,i}=E_S[y_{t,i}]$, where $E_S$ is the student's input embedding matrix. We map each input type to a common dimension $d_c$ using separate RMSNorm~\citep{zhang2019rmsnorm} and linear projection module $P$:
\begin{equation}
\widetilde h_t
=P_{\mathrm h}\operatorname{RMSNorm}_{\mathrm h}(h_t).
\end{equation}
\begin{equation}
\widetilde h_t^{\mathrm{mid}}
=P_{\mathrm m}\operatorname{RMSNorm}_{\mathrm m}(h_t^{\mathrm{mid}}).
\end{equation}
\begin{equation}
\widetilde e_{t,i} =P_{\mathrm e}\operatorname{RMSNorm}_{\mathrm e}(e_{t,i}),
\qquad i=1,\ldots,K.
\end{equation}
Stacking the two projected hidden states and the $K$ projected candidate embeddings as rows gives
\begin{equation}
X(s_t)\!=\!
\big[
\widetilde h_t,\,
\widetilde h_t^{\mathrm{mid}},\,
\widetilde e_{t,1},\,
\ldots,\,
\widetilde e_{t,K}
\big]^\top
\!\!\in\!\mathbb R^{(K+2)\times d_c}.
\end{equation}
With the $K + 2$ embedding sequence inputted, they are processed by a bidirectional transformer with self-attention without a causal mask. Each candidate can attend to both prefix representations and all candidates, allowing its correction to depend on the context and the available alternatives. We normalize each candidate's output vector and apply the same linear layer to map it to a scalar score, yielding K scores $b_\theta(X(s_t))\in\mathbb R^K$, which we center to obtain the residual:
\begin{equation}
f_\theta(s_t) = \mathcal C\!\left(b_\theta(X(s_t))\right),
\label{eq:residual_centering}
\end{equation}
\begin{equation}
\mathcal C(x) = x - \frac1K\left(\sum_{i=1}^K x_i\right)\mathbf 1_{\mathrm{v}},
\label{eq:centering_op}
\end{equation}
where $\mathbf 1_\mathrm{v}$ is the all-ones vector.

Finally, with the obtained prediction residual, let $z_{K}(s_{t})$ denote the original student logits on the top-$K$ candidate set. The corrected candidate distribution is
\begin{equation}
q_\theta(\cdot\mid s_t)
=\operatorname{softmax}\!\left(z_{K}(s_t)+f_\theta(s_t)\right).
\label{eq:residual}
\end{equation}
Student logits provide the initial preference, while CRA predicts adjustments for each candidate.

The training objective of CRA consists of two complementary parts. First, we minimize the forward KL divergence between the teacher distribution $q_T$ and the corrected distribution $q_\theta$ to align their candidate probabilities. Second, we directly supervise the predicted logit residual, providing an explicit target for the adjustments that CRA should learn. The ideal residual of $f_{\theta}(s_{t})$ is
\begin{equation}
\delta^\star(s_t)
=\mathcal C\!\left(
\log q_T(\cdot\mid s_t)-\log q_S(\cdot\mid s_t)
\right).
\label{eq:ideal_residual}
\end{equation}
We therefore optimize $f_{\theta}(s_{t})$ to match $\delta^\star(s_t)$, using a Smooth L1 loss $\mathcal{L}_{s_{t}}(\theta)$ \citep{Girshick_2015_ICCV}, with its formulation provided in Appendix~\ref{app:sra-exp-training}.

The final objective combines probability matching with direct residual supervision:
\begin{equation}
\mathcal L_{\mathrm{CRA}}(s_t)
={}D_{\mathrm{KL}}\!\left(
q_T\,\|\,q_\theta \right) +\lambda_R\mathcal L_{s_{t}}(\theta),
\label{eq:cra_loss}
\end{equation}
where $\lambda_R$ weights the residual loss.

\subsection{Joint Training}
\label{sec:two_stage}
During inference, combining DG and CRA defines our generation policy $\pi_{\phi,\theta}$, which consists of the base student model $\pi_{S}$ and two extra modules, i.e., the indicator $g_{\phi}$ and the corrector $f_{\theta}$. Concretely, under prefix $s_{t}$, when $g_\phi(s_{t})=0$, the next token is sampled from the full student distribution $\pi_S(\cdot\mid s_{t})$. When $g_\phi(s_{t})=1$, it is sampled from $q_\theta(\cdot\mid s_{t})$ corrected by $f_{\theta}(s_{t})$ over the candidate set $\mathcal K(s_{t})$ obtained by $\pi_{S}(\cdot \mid s_{t})$. Candidate correction changes the generated token and the subsequent prefixes. To account for these changes, we first learn both modules on original student trajectories, then continue training CRA on trajectories generated by the current enhanced policy. Appendix~\ref{app:sra-exp-training} summarizes the procedure.

\paragraph{Stage 1: Cold Start.} 
We generate responses offline under $\pi_S$ and compute the teacher distributions and labels at each response prefix. We train DG on these trajectories by minimizing $\mathcal L_{\mathrm{DG}}$ \eqref{eq:gate_loss} over all valid response positions. Meanwhile, we train CRA by minimizing $\mathcal L_{\mathrm{CRA}}$ \eqref{eq:cra_loss} only where $\ell(s_t)=1$ indicated by teacher supervision. After this stage, we fix the gate threshold $\eta$ in \eqref{eq:gate} and the gate parameters $\phi^\star$.

\paragraph{Stage 2: On-Policy Correction.} After the training on offline data, we train the model in an online manner to avoid the potential train-inference distribution shift \citep{agarwal2024onpolicy,wang2025improved}. 
We generate new responses with the current policy $\pi_{\phi^\star, \theta}$ and update CRA module $f_{\theta}$ at positions where $g_{\phi^\star}(s_t)=1$ by minimizing $\mathcal{L}_{\rm CRA}$ \eqref{eq:cra_loss}. These include false positive selections, allowing the adapter to learn from inputs it encounters when actually invoked. After each round of updates, we generate new responses with the updated CRA and repeat until the training budget is exhausted.

\begin{remark}
    In our experiments, SRA brings larger gains when the student is already well aligned with the teacher, presumably because a lightweight module cannot represent large and diverse corrections. We therefore build SRA on top of distilled students, and Section~\ref{sec:main-results} shows that it combines well with mainstream distillation methods.
\end{remark}

\subsection{Theoretical Analysis}
\label{sec:theory}
We show that SRA provably improves the student when its corrections are accurate enough. Let $J(\pi)$ denote the expected task reward of a policy $\pi$, and let $\mathcal E_{\mathrm{SRA}}$ denote the accumulated approximation error of SRA (see Appendix~\ref{app:sra-error-definition} for details).
\begin{theorem}
\label{thm:sra_improvement}
If the task reward lies in $[0,1]$, then
\begin{equation}
J(\pi_{\phi,\theta})-J(\pi_S)
\geq J(\pi_T)-J(\pi_S) -\mathcal E_{\mathrm{SRA}}.
\label{eq:sra_sequence_bound}
\end{equation}
\end{theorem}
Hence SRA improves the student whenever $\mathcal E_{\mathrm{SRA}}$ is smaller than the teacher's advantage. The proof is given in Appendix~\ref{app:return-bound}.

\begin{table*}[!t]
\centering\begingroup
\fontsize{9}{9.5}\selectfont
\setlength{\tabcolsep}{1.5pt}
\renewcommand{\arraystretch}{1.0}
\setlength{\aboverulesep}{1.4pt}
\setlength{\belowrulesep}{1.8pt}
\setlength{\dimen0}{\dimexpr\textwidth-20\tabcolsep\relax}
\begin{tabular}{@{}>{\raggedright\arraybackslash}p{.14\dimen0}*{8}{>{\raggedleft\arraybackslash}p{.08\dimen0}}*{2}{>{\raggedleft\arraybackslash}p{.11\dimen0}}@{}}
\toprule
 & \multicolumn{2}{c}{AIME24} & \multicolumn{2}{c}{AIME25} & \multicolumn{2}{c}{AMC23} & \multicolumn{2}{c}{MATH-500} & \multicolumn{2}{c}{Average} \\
\cmidrule(lr){2-3}\cmidrule(lr){4-5}\cmidrule(lr){6-7}\cmidrule(lr){8-9}\cmidrule(l){10-11}
Method & Avg@16 & Pass@16 & Avg@16 & Pass@16 & Avg@16 & Pass@16 & Avg@16 & Pass@16 & Avg@16 & Pass@16 \\
\midrule
\multicolumn{11}{c}{\textit{Qwen3-0.6B-Base}} \\
\addlinespace[0.4pt]
KD & 1.46 & \textbf{13.33} & 0.21 & 3.33 & 16.72 & 62.50 & 36.91 & 76.40 & 13.82 & 38.89 \\
\rowcolor{black!6}
\quad w/~SRA & 1.46 & \underline{10.00} & 0.21 & 3.33 & 17.97 & 62.50 & 37.41 & 77.20 & 14.26\textsuperscript{\textcolor{gain}{+0.44}} & 38.26\textsuperscript{\textcolor{drop}{-0.63}} \\
\addlinespace[0.6pt]
OPD & \underline{1.67} & \underline{10.00} & 0.00 & 0.00 & 25.31 & \underline{67.50} & 43.24 & 76.20 & 17.55 & 38.43 \\
\rowcolor{black!6}
\quad w/~SRA & \textbf{2.08} & \underline{10.00} & \underline{0.83} & 6.67 & \textbf{26.56} & \underline{67.50} & \textbf{47.70} & \underline{80.80} & \textbf{19.29}\textsuperscript{\textcolor{gain}{+1.74}} & 41.24\textsuperscript{\textcolor{gain}{+2.81}} \\
\addlinespace[0.6pt]
EOPD & \underline{1.67} & \underline{10.00} & \underline{0.83} & \underline{13.33} & \textbf{26.56} & 65.00 & 47.16 & 78.60 & 19.06 & \underline{41.73} \\
\rowcolor{black!6}
\quad w/~SRA & \textbf{2.08} & \underline{10.00} & \textbf{1.04} & \textbf{16.67} & \underline{26.25} & \textbf{70.00} & \underline{47.23} & \textbf{81.80} & \underline{19.15}\textsuperscript{\textcolor{gain}{+0.09}} & \textbf{44.62}\textsuperscript{\textcolor{gain}{+2.89}} \\
\midrule
\multicolumn{11}{c}{\textit{Qwen3-1.7B-Base}} \\
\addlinespace[0.4pt]
KD & 3.75 & 23.33 & 4.17 & 16.67 & 33.44 & \underline{80.00} & 57.85 & 88.00 & 24.80 & 52.00 \\
\rowcolor{black!6}
\quad w/~SRA & 5.63 & 23.33 & 2.92 & 20.00 & 34.84 & \textbf{85.00} & 59.66 & \underline{90.00} & 25.76\textsuperscript{\textcolor{gain}{+0.96}} & 54.58\textsuperscript{\textcolor{gain}{+2.58}} \\
\addlinespace[0.6pt]
OPD & 7.50 & 20.00 & 3.75 & \underline{23.33} & 39.69 & \underline{80.00} & \underline{65.64} & 89.80 & 29.14 & 53.28 \\
\rowcolor{black!6}
\quad w/~SRA & \textbf{8.33} & \textbf{30.00} & \textbf{5.00} & 20.00 & \underline{40.78} & \textbf{85.00} & \textbf{66.41} & \textbf{90.20} & \textbf{30.13}\textsuperscript{\textcolor{gain}{+0.99}} & \underline{56.30}\textsuperscript{\textcolor{gain}{+3.02}} \\
\addlinespace[0.6pt]
EOPD & 6.25 & \underline{26.67} & \underline{4.79} & \underline{23.33} & \textbf{40.94} & 77.50 & 65.43 & 89.20 & \underline{29.35} & 54.18 \\
\rowcolor{black!6}
\quad w/~SRA & \underline{8.12} & \textbf{30.00} & 4.58 & \textbf{26.67} & 39.53 & \underline{80.00} & 65.12 & 89.60 & 29.34\textsuperscript{\textcolor{drop}{-0.01}} & \textbf{56.57}\textsuperscript{\textcolor{gain}{+2.39}} \\
\midrule
\multicolumn{11}{c}{\textit{Qwen3-4B-Base}} \\
\addlinespace[0.4pt]
KD & 15.00 & \underline{40.00} & 12.08 & 30.00 & 52.97 & 85.00 & 75.71 & 93.60 & 38.94 & 62.15 \\
\rowcolor{black!6}
\quad w/~SRA & 15.00 & 36.67 & 13.13 & 36.67 & 52.66 & \textbf{90.00} & 75.91 & \underline{94.40} & 39.17\textsuperscript{\textcolor{gain}{+0.23}} & 64.43\textsuperscript{\textcolor{gain}{+2.28}} \\
\addlinespace[0.6pt]
OPD & 14.79 & 36.67 & \textbf{15.83} & 33.33 & \underline{54.22} & \textbf{90.00} & \underline{78.00} & 94.20 & 40.71 & 63.55 \\
\rowcolor{black!6}
\quad w/~SRA & \underline{16.04} & \underline{40.00} & 14.38 & \textbf{46.67} & \textbf{56.88} & \textbf{90.00} & \textbf{78.49} & \textbf{95.00} & \textbf{41.45}\textsuperscript{\textcolor{gain}{+0.74}} & \underline{67.92}\textsuperscript{\textcolor{gain}{+4.37}} \\
\addlinespace[0.6pt]
EOPD & 15.42 & \underline{40.00} & 14.58 & \underline{43.33} & 54.06 & \underline{87.50} & 77.78 & 93.60 & 40.46 & 66.11 \\
\rowcolor{black!6}
\quad w/~SRA & \textbf{17.71} & \textbf{43.33} & \underline{14.79} & \textbf{46.67} & 54.06 & \textbf{90.00} & 77.89 & 93.60 & \underline{41.11}\textsuperscript{\textcolor{gain}{+0.65}} & \textbf{68.40}\textsuperscript{\textcolor{gain}{+2.29}} \\
\midrule
\textit{Qwen3-8B} & 25.63 & 63.33 & 20.42 & 43.33 & 67.50 & 92.50 & 82.75 & 95.40 & 49.07 & 73.64 \\
\bottomrule
\end{tabular}\endgroup
\caption{Main results of distilled Qwen3 models on mathematical reasoning benchmarks. Shaded rows use SRA. Superscripts show changes from the corresponding baseline in percentage points, with gains in green and drops in red. Boldface and underlining indicate the best and second best results among student models of the same size.}
\label{tab:main-results}
\vspace{-3mm}
\end{table*}

\defcitealias{maa_aime}{MAA, 2025}
\defcitealias{maa_amc}{MAA, 2023}

\section{Experiments}
\label{sec:experiments}
Our experiments are designed to address three key questions: whether SRA consistently benefits students trained with diverse distillation methods (Section~\ref{sec:main-results}), whether its gains stem from more than extended training (Section~\ref{sec:continued-distillation}), and how each component contributes to the overall gains (Section~\ref{sec:main-ablations}).

\subsection{Main Results}
\label{sec:main-results}
We start by verifying that our SRA can improve the existing distillation-based alignment methods. 

\paragraph{Setup.} Our experiments are conducted on Qwen3 models~\citep{yang2025qwen3} across 0.6B, 1.7B, and 4B scales. As clarified in Section \ref{sec:two_stage}, our SRA can achieve better performance when student and teacher models are aligned better. Thus we build our SRA upon three mainstream distillation methods, Knowledge Distillation (KD)~\citep{hinton2015distilling,kim2016sequence}, On-Policy Distillation (OPD)~\citep{lu2025onpolicydistillation}, and Entropy-Aware On-Policy Distillation (EOPD)~\citep{jin2026entropy} as baselines.  We also evaluate SRA with the non-thinking mode of a Qwen3 model as the student and its thinking mode as the teacher. Full experimental settings and results for these and the LLaMA~\citep{meta2024llama32} experiments are provided in Appendices~\ref{app:cross-mode} and~\ref{app:LLaMA-results}.

We use mathematical reasoning dataset DAPO-Math-17k~\citep{yu2025dapo} as the training set across all experiments. We first train the student models using KD, OPD, or EOPD for two epochs. Starting from these students, we train SRA with the student parameters frozen for two additional epochs. Detailed configurations for the distillation baselines and SRA are provided in Appendix~\ref{app:backbone-preparation}.

To evaluate the performance, we sample 16 responses per question on mathematical benchmarks  AIME24/25~\citepalias{maa_aime}, AMC23~\citepalias{maa_amc}, and MATH-500~\citep{lightman2024verify}, using temperature=1, top-$p=1$, and a maximum response length of 31,744 tokens. We report Avg@16 and Pass@16. The evaluation prompt is provided in Appendix~\ref{app:prompts}. We also evaluate out-of-domain generalization, with results reported in Appendix~\ref{app:ood-evaluation}.

\paragraph{Results.} Table~\ref{tab:main-results} shows that SRA improves the average Pass@16 in eight of the nine settings, by 2.28 to 4.37 points, and its gains in Pass@16 are clearly larger than those in Avg@16. This asymmetric advantage demonstrates that selective correction curtails compounding errors during autoregressive generation while preserving output diversity, substantially expanding solution exploration. SRA displays remarkable synergy with on-policy distillation methods such as OPD and EOPD, achieving the best  Avg@16 and Pass@16 at 4B and surpassing the 8B teacher on AIME25 Pass@16. Furthermore, Appendix~\ref{app:teacher-alignment} shows that SRA improves alignment with the teacher's distribution over the student's top-$K$ set.

\subsection{Boosting on Distillation Methods}
\label{sec:continued-distillation}
The results above are obtained by continually training on distillation methods, indicating that our SRA effectively boosts these methods. However, a fair comparison requires comparing it with continual distillation by the original alignment methods. Therefore, we continue the distillation methods for another two epochs and compare them with SRA.

\paragraph{Results.} Table~\ref{tab:continued-distillation} summarizes the comparison between SRA and continued distillation on Qwen3-4B. SRA achieves consistent improvements across all baselines, boosting average Avg@16 and Pass@16 by 2.32 and 1.80 points on KD, while outperforming continued distillation on OPD and EOPD by 4.21 and 3.12 points in Pass@16, respectively. Simply extending distillation steps yields diminishing marginal returns, whereas SRA overcomes this performance plateau. 

This advantage stems from fundamental differences in optimization mechanics. Standard distillation relies on a global loss averaged across all generation steps, diluting learning signals at crucial discrepancy points once the student saturates. In contrast, SRA selectively intervenes only at states with high distribution divergence via the Discrepancy Gate, preserving generative diversity elsewhere. By applying localized residual corrections directly to the candidate set, SRA effectively boosts fully converged distillation models without perturbing the global representation space.

\begin{table}[t]
\centering\begingroup
\fontsize{9}{9.5}\selectfont
\setlength{\tabcolsep}{0.8pt}
\renewcommand{\arraystretch}{1.0}
\setlength{\aboverulesep}{1.4pt}
\setlength{\belowrulesep}{1.8pt}
\setlength{\dimen0}{\dimexpr\columnwidth-10\tabcolsep\relax}
\begin{tabular}{@{}>{\raggedright\arraybackslash\leavevmode}p{.204\dimen0}>{\raggedright\arraybackslash\leavevmode}p{.133\dimen0}>{\raggedleft\arraybackslash\leavevmode}p{.153\dimen0}>{\raggedleft\arraybackslash\leavevmode}p{.158\dimen0}>{\raggedleft\arraybackslash\leavevmode}p{.183\dimen0}>{\raggedleft\arraybackslash\leavevmode}p{.169\dimen0}@{}}
\toprule
\multirow{2}{*}{Benchmark} & \multirow{2}{*}{Method} & \multicolumn{2}{c}{Student} & \multicolumn{2}{c}{\cellcolor{black!6}w/~SRA} \\
\cmidrule(lr){3-4}\cmidrule(l){5-6}
 & & Avg@16 & Pass@16 & \cellcolor{black!6}Avg@16 & \cellcolor{black!6}Pass@16 \\
\midrule
\multirow{4}{*}{AIME24} & Base & 3.54 & 26.67 & \cellcolor{black!6}10.42 & \cellcolor{black!6}23.33 \\
 & KD & 12.92 & 30.00 & \cellcolor{black!6}15.00 & \cellcolor{black!6}36.67 \\
 & OPD & 15.83 & 36.67 & \cellcolor{black!6}16.04 & \cellcolor{black!6}\underline{40.00} \\
 & EOPD & \underline{16.67} & \underline{40.00} & \cellcolor{black!6}\textbf{17.71} & \cellcolor{black!6}\textbf{43.33} \\
\midrule
\multirow{4}{*}{AIME25} & Base & 2.71 & 16.67 & \cellcolor{black!6}7.50 & \cellcolor{black!6}26.67 \\
 & KD & \textbf{15.83} & \underline{40.00} & \cellcolor{black!6}13.13 & \cellcolor{black!6}36.67 \\
 & OPD & \underline{15.00} & 36.67 & \cellcolor{black!6}14.38 & \cellcolor{black!6}\textbf{46.67} \\
 & EOPD & 14.58 & \underline{40.00} & \cellcolor{black!6}14.79 & \cellcolor{black!6}\textbf{46.67} \\
\midrule
\multirow{4}{*}{AMC23} & Base & 15.00 & 70.00 & \cellcolor{black!6}39.69 & \cellcolor{black!6}82.50 \\
 & KD & 47.34 & \underline{87.50} & \cellcolor{black!6}52.66 & \cellcolor{black!6}\textbf{90.00} \\
 & OPD & \textbf{56.88} & \underline{87.50} & \cellcolor{black!6}\textbf{56.88} & \cellcolor{black!6}\textbf{90.00} \\
 & EOPD & 53.75 & \underline{87.50} & \cellcolor{black!6}\underline{54.06} & \cellcolor{black!6}\textbf{90.00} \\
\midrule
\multirow{4}{*}{MATH-500} & Base & 20.94 & 88.00 & \cellcolor{black!6}64.84 & \cellcolor{black!6}92.40 \\
 & KD & 71.29 & 93.00 & \cellcolor{black!6}75.91 & \cellcolor{black!6}\underline{94.40} \\
 & OPD & \underline{78.04} & 94.00 & \cellcolor{black!6}\textbf{78.49} & \cellcolor{black!6}\textbf{95.00} \\
 & EOPD & 77.40 & 93.60 & \cellcolor{black!6}77.89 & \cellcolor{black!6}93.60 \\
\midrule
\multirow{4}{*}{Average} & Base & 10.55 & 50.33 & \cellcolor{black!6}30.61\textsuperscript{\textcolor{gain}{+20.06}} & \cellcolor{black!6}56.23\textsuperscript{\textcolor{gain}{+5.90}} \\
 & KD & 36.85 & 62.63 & \cellcolor{black!6}39.17\textsuperscript{\textcolor{gain}{+2.32}} & \cellcolor{black!6}64.43\textsuperscript{\textcolor{gain}{+1.80}} \\
 & OPD & \underline{41.44} & 63.71 & \cellcolor{black!6}\textbf{41.45}\textsuperscript{\textcolor{gain}{+0.01}} & \cellcolor{black!6}\underline{67.92}\textsuperscript{\textcolor{gain}{+4.21}} \\
 & EOPD & 40.60 & 65.28 & \cellcolor{black!6}41.11\textsuperscript{\textcolor{gain}{+0.51}} & \cellcolor{black!6}\textbf{68.40}\textsuperscript{\textcolor{gain}{+3.12}} \\
\bottomrule
\end{tabular}\endgroup
\caption{Comparison of SRA and continued distillation on Qwen3-4B. Student extends distillation by two epochs, shaded applies SRA after two epochs, and Base applies SRA directly to the base model.}
\label{tab:continued-distillation}
\label{tab:continued-opd}
\vspace{-3mm}
\end{table}

\subsection{Ablation Studies}
\label{sec:main-ablations}
\label{sec:component-ablations}
In this section, we conduct ablation studies of SRA to examine the contributions of the student's final hidden state, intermediate hidden state, and residual supervision. Additional experiments on candidate set size $K$, training strategy, and module architecture are provided in Appendix~\ref{app:additional-comparisons}.
\paragraph{Final hidden state.}
We remove the final hidden state from CRA's input while retaining the intermediate hidden state from L24 and the candidate token embeddings. In Table~\ref{tab:component-ablations}, Avg@16 and Pass@16 drop by 0.57 and 1.76 points, respectively, confirming that the final representation driving DG in Section~\ref{sec:dg} provides complementary signals not captured by the other two inputs.
\paragraph{Intermediate hidden state.}
Following the multilayer feature fusion in Section~\ref{sec:cra} inspired by EAGLE~\citep{li2025eagle3}, CRA incorporates an intermediate hidden state. Removing this intermediate state causes the sharpest decline across all components, lowering Avg@16 and Pass@16 by 1.04 and 4.64 points, respectively. When all three inputs are retained, extracting the intermediate representation from L24 consistently outperforms L12 and L32 across both metrics, validating the fusion of final and intermediate representations.
\paragraph{Residual supervision.}
Setting the Smooth L1 loss weight $\lambda_R$ to 0 in Eq.~\eqref{eq:cra_loss} ablates the regression of $f_\theta$ toward the ideal residual in Eq.~\eqref{eq:ideal_residual} while keeping the KL coefficient at 1, which reduces Avg@16 and Pass@16 by 0.92 and 1.87 points. Furthermore, increasing $\lambda_R$ from the default 0.1 to 1 also degrades performance, confirming that balanced residual supervision effectively complements probability distribution matching.

\begin{table}[t]
\centering\begingroup
\fontsize{9}{9.5}\selectfont
\setlength{\tabcolsep}{2pt}
\renewcommand{\arraystretch}{1.0}
\setlength{\dimen0}{\dimexpr\columnwidth-4\tabcolsep\relax}
\begin{tabular}{@{}>{\raggedright\arraybackslash\leavevmode}p{.64\dimen0}*{2}{>{\raggedleft\arraybackslash\leavevmode}p{.18\dimen0}}@{}}
\toprule
Setting & Avg@16 & Pass@16 \\
\midrule
\rowcolor{black!6}
Full SRA (L24, $\lambda_R=0.1$) & \textbf{41.45} & \textbf{67.92} \\
\midrule
\multicolumn{3}{l}{\textit{Final hidden state}} \\
w/o final hidden state & \underline{40.88} & \underline{66.16} \\
\midrule
\multicolumn{3}{l}{\textit{Intermediate hidden state}} \\
w/o intermediate state & 40.41 & 63.28 \\
L12 & 41.03 & 64.60 \\
L32 & \underline{41.25} & \underline{66.21} \\
\midrule
\multicolumn{3}{l}{\textit{Residual supervision}} \\
$\lambda_R=0$ & 40.53 & \underline{66.05} \\
$\lambda_R=1$ & \underline{41.05} & 65.28 \\
\bottomrule
\end{tabular}\endgroup
\caption{Ablations on Qwen3-4B OPD across four mathematical reasoning benchmarks against SRA. Full results and configurations are in Appendix~\ref{app:component-ablations}.}
\label{tab:component-ablations}
\vspace{-3mm}
\end{table}

\section{Conclusion}
In this work, we present Selective Residual Adaptation (SRA), a plug-and-play framework that boosts distilled student reasoning without backbone updates or inference-time teacher queries. Exploiting the sparsity of teacher-student divergence, SRA selectively calibrates candidate token distributions via a Discrepancy Gate and a Candidate Residual Adapter. Supported by rigorous theoretical analysis and extensive experiments, we show that SRA consistently strengthens student policies and preserves general capabilities.

\section*{Limitations}
While SRA achieves consistent reasoning gains, several notable limitations remain. First, its formulation relies on the assumption that teacher-preferred tokens reside within the student top-$K$ candidate set, whose coverage may decrease in unconstrained open-domain generation. Second, although out-of-domain generalization has been validated on coding and multi-disciplinary benchmarks, our evaluation primarily targets structured reasoning tasks, leaving open-ended generation such as free-form dialogue for our future investigation. Finally, although SRA eliminates runtime teacher queries during inference, its two stage training framework requires dynamic on-policy rollouts and paired teacher forward passes, incurring additional training compute overhead.

\bibliography{custom}

\clearpage

\appendix

\section{Empirical Motivation Experiments}
\label{app:empirical-motivation}

We select 256 questions from the GSM8K test split~\citep{cobbe2021gsm8k} and compare Qwen3-0.6B-Base and Qwen3-1.7B-Base with their OPD counterparts~\citep{yang2025qwen3}. Qwen3-8B serves as the teacher with thinking mode disabled.

\subsection{Training setup}
\label{app:opd-checkpoints}
We obtain the OPD students by training the Base checkpoints for two epochs on DAPO-Math-17k~\citep{yu2025dapo}. We adopt the OPD baseline configuration~\citep{jin2026entropy}, with a maximum response length of 7168 tokens and the training prompt below. We optimize with AdamW~\citep{loshchilov2019adamw}. Table~\ref{tab:opd-hparams} summarizes the hyperparameters.

\begin{table*}[ht]
\centering
\begingroup
\fontsize{9}{10}\selectfont
\setlength{\tabcolsep}{3pt}
\renewcommand{\arraystretch}{1.05}

\begin{tabular*}{\textwidth}{@{\extracolsep{\fill}}lcccc@{}}
\toprule
Hyperparameter & KD & OPD & EOPD & SRA \\
\midrule
\multicolumn{5}{l}{\textit{Training configuration}} \\
Optimizer
& AdamW & AdamW & AdamW & AdamW \\
Learning rate
& $10^{-5}$ & $3\times10^{-6}$ & $3\times10^{-6}$
& DG: $10^{-4}$, CRA: $5\times10^{-5}$ \\
Learning rate schedule
& Cosine & Cosine & Cosine & Constant \\
Training epochs
& 2 & 2 & 2 & 2 \\
Batch size
& 128 & 128 & 128 & 128 \\
Minibatch size
& 32 & 32 & 32 & 32 \\
Responses per prompt
& 1 & 1 & 1 & 1 \\
Maximum prompt tokens
& 1,024 & 1,024 & 1,024 & 1,024 \\
Maximum response tokens
& 7,168 & 7,168 & 7,168 & 7,168 \\
Temperature / top-$p$
& 1 / 1 & 1 / 1 & 1 / 1 & 1 / 1 \\
Top-$k$ / candidate size
& 16 & 16 & 16 & 16 \\

\midrule
\multicolumn{5}{l}{\textit{Objectives and method-specific settings}} \\
CE coefficient
& 0.5 & -- & -- & -- \\
Forward KL coefficient
& 0.5 & -- & 1.0 & 1 \\
Teacher entropy threshold
& -- & -- & 0.8 & -- \\
TV threshold $\tau$
& -- & -- & -- & 0.1 \\
Smooth L1 coefficient $\lambda_R$
& -- & -- & -- & 0.1 \\
Smooth L1 $\beta$
& -- & -- & -- & 1 \\
Residual target clip
& -- & -- & -- & $[-12,12]$ \\
CRA Transformer layers
& -- & -- & -- & 2 \\
DG optimization passes per rollout batch (epoch 1 / 2)
& -- & -- & -- & 2 / 0 \\
CRA optimization passes per rollout batch (epoch 1 / 2)
& -- & -- & -- & 1 / 1 \\
\bottomrule
\end{tabular*}

\endgroup
\caption{Training hyperparameters for KD, OPD, EOPD, and SRA.}
\label{tab:training-settings}
\label{tab:kd-settings}
\label{tab:onpolicy-settings}
\label{tab:sra-settings}
\label{tab:opd-hparams}
\end{table*}

The training prompt is given below, with \texttt{\{problem\}} replaced by the input question.

\begin{sratrainingprompt}
Solve the following math problem step by step.
The last line of your response should be of the form
Answer: \texttt{\$Answer} (without quotes) where
\texttt{\$Answer} is the answer to the problem.
\par\medskip
\texttt{\{problem\}}
\par\medskip
Remember to put your answer on its own line after
\texttt{"Answer:"}.
\end{sratrainingprompt}

\subsection{Evaluation Protocol}
\label{app:motivation-details}
\label{app:motivation-settings}
\label{app:prompts}

We randomly sample 256 distinct questions from the GSM8K test split~\citep{cobbe2021gsm8k}. This question set is shared across all conditions and is not used for OPD training. For each question and condition, we sample one response with temperature 1.0, top-$p=1.0$, and a maximum of 512 new tokens. All models are evaluated in non-thinking mode using the following prompt:

\begin{sraevaluationprompt}
\texttt{\{problem\}} Please reason step by step, and put your final answer within \texttt{\textbackslash boxed\{\}}.
\end{sraevaluationprompt}

Answer accuracy is computed by comparing the final boxed answer with the reference answer. Responses without a boxed answer are scored as incorrect. All training and evaluation runs are conducted on a single machine equipped with 8 NVIDIA A100 GPUs.

\subsection{TV Threshold Sensitivity}
\label{app:threshold-sensitivity}

We vary $\tau$ from 0.06 to 0.12 in increments of 0.01, holding $K=16$ and the original student responses fixed. Figure~\ref{fig:tv-threshold} reports disagreement recall and position selection rate.

At $\tau=0.1$, recall exceeds 98\% in all four settings, while the selection rate is 3.37--6.82 percentage points lower than at $\tau=0.06$. The 0.6B OPD student is particularly sensitive to further increases in the threshold. Its recall decreases from 98.03\% at $\tau=0.1$ to 97.21\% and 96.39\% at $\tau=0.11$ and $\tau=0.12$, respectively, while the selection rate decreases by an additional 1.77 and 3.59 percentage points. Thus, higher thresholds reduce the number of selected positions but miss more disagreements. Among the evaluated thresholds, $\tau=0.1$ is the largest value that maintains recall above 98\% across all four settings, supporting its use as a practical threshold for retaining high disagreement recall while limiting position selection.

\begin{figure*}[ht]
\centering
\includegraphics[width=\textwidth]{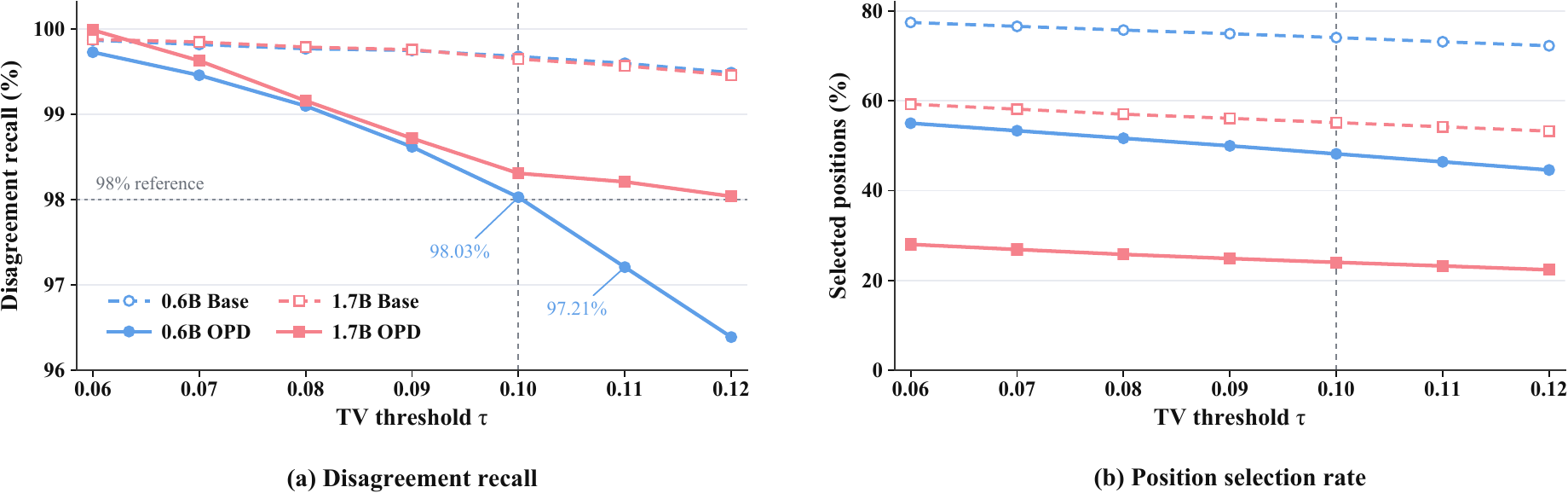}
\caption{TV threshold sensitivity on GSM8K with $K=16$. (a) Disagreement recall. (b) Proportion of selected positions. Vertical lines mark $\tau=0.1$.}
\label{fig:tv-threshold}
\end{figure*}
\section{Theoretical Analysis}
\label{app:sra-theory}
\label{app:proofs}

\subsection{Local Sensitivity of TV and KL}
\label{app:tv-local-sensitivity}

Fix a generation prefix $s_t$ and the student's candidate set $\mathcal K(s_t)=\{y_{t,1},\ldots,y_{t,K}\}$. Let $q_T=(q_{T,1},\ldots,q_{T,K})$ denote the teacher distribution normalized on this set, where $q_{T,i}=q_T(y_{t,i}\mid s_t)$. We keep this candidate set fixed throughout the local analysis and suppress $s_t$ below.

To describe student distributions near $q_T$, consider
\begin{equation}
q_S^{(\varepsilon)}=q_T+\varepsilon\Delta,
\label{eq:app-local-perturbation}
\end{equation}
where $\Delta=(\Delta_1,\ldots,\Delta_K)$ is a fixed nonzero perturbation vector. A positive $\Delta_i$ increases the student's probability of candidate $i$ relative to the teacher, and a negative $\Delta_i$ decreases it. The scalar $\varepsilon>0$ scales these changes. The condition $\sum_i\Delta_i=0$ preserves the total probability mass.

\begin{theorem*}[Theorem~1, Local sensitivity of TV and KL]
\label{thm:tv-local-response-app}
Let $q_T$ be a fixed probability distribution on $K$ candidates with $q_{T,i}>0$. Let $\Delta\in\mathbb R^K\setminus\{0\}$ be fixed and satisfy $\sum_i\Delta_i=0$. For sufficiently small $\varepsilon>0$, the vector $q_S^{(\varepsilon)}=q_T+\varepsilon\Delta$ is a probability distribution and satisfies
\begin{equation}
\operatorname{TV}(q_T, q_S^{(\varepsilon)})
=\frac{\varepsilon}{2}\|\Delta\|_1.
\label{eq:app-tv-local-expansion}
\end{equation}
Using natural logarithms, the two KL directions satisfy
\begin{equation}
D_{\mathrm{KL}}(q_T\|q_S^{(\varepsilon)})
=\frac{\varepsilon^2}{2}\sum_{i=1}^K\frac{\Delta_i^2}{q_{T,i}}
+O(\varepsilon^3),
\label{eq:app-fkl-local-expansion}
\end{equation}
and
\begin{equation}
D_{\mathrm{KL}}(q_S^{(\varepsilon)}\|q_T)
=\frac{\varepsilon^2}{2}\sum_{i=1}^K\frac{\Delta_i^2}{q_{T,i}}
+O(\varepsilon^3).
\label{eq:app-rkl-local-expansion}
\end{equation}
Thus, along the same perturbation direction, TV has a first-order response and both KL directions have second-order responses.
\end{theorem*}

\begin{proof}
Choose
\begin{equation}
0<\varepsilon<\frac12\min_{i:\Delta_i\ne0}
\frac{q_{T,i}}{|\Delta_i|}.
\end{equation}
The upper bound is strictly positive because $\Delta\ne0$ and each $q_{T,i}>0$. Every coordinate of $q_S^{(\varepsilon)}$ is then positive. Moreover, $\sum_i\Delta_i=0$ ensures that the coordinates sum to one.
The definition of TV gives:
\begin{equation}
\begin{aligned}
\operatorname{TV}(q_T, q_S^{(\varepsilon)})
&=\frac12\sum_{i=1}^K
|q_{T,i}+\varepsilon\Delta_i-q_{T,i}|\\
&=\frac{\varepsilon}{2}\|\Delta\|_1.
\end{aligned}  
\end{equation}
Set $u_i=\varepsilon\Delta_i/q_{T,i}$, so $|u_i|<1/2$. Using
\begin{equation}
\log(1+u)=u-\frac12u^2+O(u^3),
\end{equation}
the forward KL becomes
\begin{equation}
\begin{aligned}
D_{\mathrm{KL}}(q_T\|q_S^{(\varepsilon)})
&=-\sum_iq_{T,i}\log(1+u_i)\\
&=-\varepsilon\sum_i\Delta_i +\frac{\varepsilon^2}{2}\sum_i\frac{\Delta_i^2}{q_{T,i}}
+O(\varepsilon^3)\\
&=\frac{\varepsilon^2}{2}\sum_i\frac{\Delta_i^2}{q_{T,i}}
+O(\varepsilon^3).
\end{aligned}    
\end{equation}
For the reverse direction, the expansion
\begin{equation}
(1+u)\log(1+u)=u+\frac12u^2+O(u^3)
\end{equation}
yields
\begin{equation}
\begin{aligned}
D_{\mathrm{KL}}(q_S^{(\varepsilon)}\|q_T)
&=\sum_iq_{T,i}(1+u_i)\log(1+u_i)\\
&=\varepsilon\sum_i\Delta_i +\frac{\varepsilon^2}{2}\sum_i\frac{\Delta_i^2}{q_{T,i}}
+O(\varepsilon^3)\\
&=\frac{\varepsilon^2}{2}\sum_i\frac{\Delta_i^2}{q_{T,i}}
+O(\varepsilon^3).
\end{aligned}
\end{equation}
The remainder is $O(\varepsilon^3)$ because the candidate set is finite and $q_T,\Delta$ are fixed. Since $\Delta\ne0$, both $\|\Delta\|_1$ and $\sum_i\Delta_i^2/q_{T,i}$ are strictly positive. Therefore, the leading nonzero term is first order for TV and second order for both KL directions.
\end{proof}

\subsection{Setup and Approximation Error}
\label{app:sra-error-definition}
\label{app:local-policy-error}

All policies share a question distribution $\mathcal D$, a finite output vocabulary $\mathcal V$, and a maximum generation horizon $H$. For a terminal reward $R(x,y)\in[0,1]$, define
\begin{equation}
J(\pi)=\mathbb E_{x\sim\mathcal D,\,y\sim\pi(\cdot\mid x)}[R(x,y)].
\label{eq:expected_reward}
\end{equation}
After termination, all policies take the same absorbing action, and gate activation and error contributions are set to zero. The notation $\mathbb E_\pi$ includes both question sampling and trajectories generated by $\pi$.

At an active prefix $s_t=(x,y_{<t})$, let $\mathcal K(s_t)$ be the student's top-$K$ candidate set. For $M\in\{S,T\}$, define the retained probability mass
\begin{equation}
m_M(s_t)=\sum_{a\in\mathcal K(s_t)}\pi_M(a\mid s_t)
\label{eq:retained_mass}
\end{equation}
and the distribution conditioned on this same set
\begin{equation}
q_M(a\mid s_t)=\frac{\pi_M(a\mid s_t)}{m_M(s_t)},
\qquad a\in\mathcal K(s_t).
\label{eq:app-candidate-conditional}
\end{equation}
Candidate probabilities are strictly positive, as with finite softmax logits. The gate decision is $g_\phi(s_t)\in\{0,1\}$, and CRA produces the corrected probabilities $q_\theta(\cdot\mid s_t)$ on $\mathcal K(s_t)$. A bar denotes extension by zero outside this set. 

Writing $\pi=\pi_{\phi,\theta}$, the deployed policy is
\begin{equation}
\pi(\cdot\mid s_t)=
\begin{cases}
\pi_S(\cdot\mid s_t),&g_\phi(s_t)=0,\\
\bar q_\theta(\cdot\mid s_t),&g_\phi(s_t)=1.
\end{cases}
\label{eq:app-policy}
\end{equation}

Define the candidate set error term
\begin{equation}
\begin{aligned}
&\epsilon_K(s_t)= g_\phi(s_t)(1-m_T(s_t)) \\
&+ (1-g_\phi(s_t)) \cdot\max\{1-m_S(s_t),1-m_T(s_t)\}.
\end{aligned}
\label{eq:sra_truncation_error}
\end{equation}
The total approximation error used in Section~\ref{sec:theory} is
\begin{equation}
\begin{aligned}
\mathcal E_{\mathrm{SRA}}
:={}&\mathbb E_{\pi_{\phi,\theta}}
\sum_{t=1}^{H}\Bigl[\epsilon_K(s_t)\\
&+(1-g_\phi(s_t))\operatorname{TV}(q_S,q_T)\\
&+g_\phi(s_t)\operatorname{TV}(q_T,q_\theta)\Bigr].
\end{aligned}
\label{eq:sra-total-approximation-error}
\end{equation}
All distributions inside the sum are evaluated at $s_t$. The first term accounts for probability mass outside the candidate set, the second for discrepancies left at unselected positions, and the third for the remaining error after correction. The following lemma connects their sum to the deployed policy's distance from the teacher.

\begin{lemma}[Local policy error]
\label{lem:sra-local-error}
At every active prefix $s_t$, the policy in Equation~\eqref{eq:app-policy} satisfies
\begin{equation}
\begin{aligned}
\operatorname{TV}(\pi_T, \pi)
\leq{}&\epsilon_K+(1-g_\phi)\operatorname{TV}(q_T,q_S)\\
&+g_\phi\operatorname{TV}(q_T,q_\theta),
\end{aligned}
\label{eq:local-policy-error}
\end{equation}
where state arguments are suppressed for readability.
\end{lemma}

\begin{proof}
Fix a prefix and write $v=\operatorname{TV}(q_T,q_S)$ and $\mathcal K=\mathcal K(s_t)$. When $g_\phi=0$, the deployed distribution is $\pi_S$. Set $\alpha=\min\{m_S,m_T\}$. The overlap identity for TV gives
\begin{equation}
\begin{aligned}
1-\operatorname{TV}(\pi_T,\pi_S)
&=\sum_{a\in\mathcal V}\min\{\pi_S(a),\pi_T(a)\}\\
&\geq\alpha\sum_{a\in\mathcal K}\min\{q_S(a),q_T(a)\}\\
&=\alpha(1-v).
\end{aligned}
\end{equation}
Consequently,
\begin{equation}
\begin{aligned}
\operatorname{TV}(\pi_T,\pi_S)
&\leq1-\alpha+\alpha v\\
&\leq\max\{1-m_S,1-m_T\}+v.
\end{aligned}
\label{eq:student-support-bound}
\end{equation}

When $g_\phi=1$, the deployed distribution is $\bar q_\theta$. If $m_T<1$, decompose the teacher as
\begin{equation}
\pi_T=m_T\bar q_T+(1-m_T)\pi_T^{\mathrm{out}},
\end{equation}
where $\pi_T^{\mathrm{out}}$ is the teacher distribution conditioned on $\mathcal V\setminus\mathcal K$. Convexity of TV yields
\begin{equation}
\begin{aligned}
\operatorname{TV}(\pi_T, \bar q_\theta)
&\leq m_T\operatorname{TV}(q_T, q_\theta)+(1-m_T)\\
&\leq\operatorname{TV}(q_T, q_\theta)+(1-m_T).
\end{aligned}
\label{eq:corrected-support-bound}
\end{equation}
For $m_T=1$, the same bound follows from $\pi_T=\bar q_T$. Combining the two gate decisions proves the lemma.
\end{proof}

\subsection{Performance Improvement Guarantee}
\label{app:return-bound}

We now give the proof of Theorem~\ref{thm:sra_improvement}.

\begin{theorem*}[Theorem~2, Performance under selective correction]
\label{thm:sra-improvement-formal}
Under the setup in Appendix~\ref{app:sra-error-definition}, let $\mathcal E_{\mathrm{SRA}}$ be defined by Equation~\eqref{eq:sra-total-approximation-error}. Then
\begin{equation}
J(\pi_{\phi,\theta})\geq J(\pi_T)-\mathcal E_{\mathrm{SRA}},
\label{eq:app-teacher-performance-bound}
\end{equation}
and hence
\begin{equation}
J(\pi_{\phi,\theta})-J(\pi_S)
\geq J(\pi_T)-J(\pi_S) -\mathcal E_{\mathrm{SRA}}.
\label{eq:app-sra-improvement-bound}
\end{equation}
In particular, $\mathcal E_{\mathrm{SRA}}<J(\pi_T)-J(\pi_S)$ implies $J(\pi_{\phi,\theta})>J(\pi_S)$.
\end{theorem*}

\begin{proof}
Write $\pi=\pi_{\phi,\theta}$. Let $V_t(s_t)$ be the expected terminal reward when continuing with the teacher from state $s_t$. For a possible next token $a$, let $(s_t, a)$ be the next state and define
\begin{equation}
Q_t(s_t,a)=V_{t+1}(s_{t+1}).
\end{equation}
The teacher value satisfies
\begin{equation}
V_t(s_t)=\mathbb E_{a\sim\pi_T(\cdot\mid s_t)}Q_t(s_t,a).
\end{equation}
The append operation leaves an absorbing state unchanged. We have $V_{H + 1}=R$, so all values $V_t$ and $Q_t$ lie in $[0,1]$.

Telescoping along trajectories generated by $\pi$, and using the shared question distribution, gives
\begin{equation}
\begin{aligned}
J(\pi)-J(\pi_T)
={}&\sum_{t=1}^{H}\mathbb E_\pi
[V_{t+1}(s_{t+1})-V_t(s_t)]\\
={}&\sum_{t=1}^{H}\mathbb E_\pi\Bigl[
\mathbb E_{a\sim\pi(\cdot\mid s_t)}Q_t(s_t,a)\\
&\qquad-\mathbb E_{a\sim\pi_T(\cdot\mid s_t)}Q_t(s_t,a)\Bigr].
\end{aligned}
\label{eq:teacher-telescope}
\end{equation}
For distributions $P,Q$ and any $f$ taking values in $[0,1]$,
\begin{equation*}
\mathbb E_P[f]-\mathbb E_Q[f]\geq-\operatorname{TV}(P,Q).
\end{equation*}
Applying this inequality to each $Q_t(s_t,\cdot)$ yields
\begin{equation*}
\begin{aligned}
J(\pi)-J(\pi_T)
\geq-\mathbb E_\pi\sum_{t=1}^{H}
\operatorname{TV}(\pi,\pi_T)(s_t).
\end{aligned}
\end{equation*}
Lemma~\ref{lem:sra-local-error} and Equation~\eqref{eq:sra-total-approximation-error} bound the right-hand side below by $-\mathcal E_{\mathrm{SRA}}$. Rearranging proves Equation~\eqref{eq:app-teacher-performance-bound}. Subtracting $J(\pi_S)$ from both sides proves Equation~\eqref{eq:app-sra-improvement-bound} and the stated sufficient condition for improvement.
\end{proof}

\begin{algorithm}[ht]
\caption{SRA training}
\label{alg:sra-training}
\small
\begin{algorithmic}[1]
\Require Dataset $\mathcal D$, frozen student $\pi_S$ and teacher $\pi_T$,
candidate size $K$, TV threshold $\tau$, gate threshold $\eta$
\Ensure Trained DG parameters $\phi$ and CRA parameters $\theta$
\State Initialize $\phi$ and $\theta$

\For{each minibatch $\mathcal B$ in $\mathcal D$} \Comment{\textbf{Stage 1: Cold Start.}}
    \State $\mathcal Y \gets \operatorname{Rollout}(\pi_S,\mathcal B)$
    \State Collect response prefixes $\mathcal S$ from $(\mathcal B,\mathcal Y)$
    \ForAll{$s\in\mathcal S$}
        \State Extract student features and top-$K$ set $\mathcal K(s)$
        \State Compute $q_S,q_T$ using Equation~\eqref{eq:normalized}
        \State $\ell(s)\gets\mathbf 1[\mathrm{TV}(q_T, q_S)\geq\tau]$
    \EndFor
    \State $\phi\gets\operatorname{Update}\!\left(
    \phi,\frac{1}{|\mathcal S|}\sum_{s\in\mathcal S}
    \mathcal L_{\mathrm{DG}}(s)\right)$
    \State $\mathcal S^+\gets\{s\in\mathcal S:\ell(s)=1\}$
    \If{$\mathcal S^+\neq\varnothing$}
        \State $\theta\gets\operatorname{Update}\!\left(
        \theta,\frac{1}{|\mathcal S^+|}\sum_{s\in\mathcal S^+}
        \mathcal L_{\mathrm{CRA}}(s)\right)$
    \EndIf
\EndFor

\State \textbf{Freeze $\phi^{\star}$}
\For{each minibatch $\mathcal B$ in $\mathcal D$} \Comment{\textbf{Stage 2: On-Policy Correction.}}
    \State $\mathcal Y\gets\operatorname{Rollout}(\pi_{\phi^{\star},\theta},\mathcal B)$
    \State Collect response prefixes $\mathcal S$ from $(\mathcal B,\mathcal Y)$
    \State $\mathcal S^+\gets\{s\in\mathcal S:g_{\phi^{\star}}(s)=1\}$
    \If{$\mathcal S^+\neq\varnothing$}
        \State Extract student features and $\mathcal K(s)$ for $s\in\mathcal S^+$
        \State Compute $q_S,q_T$ using Equation~\eqref{eq:normalized}
        \State $\theta\gets\operatorname{Update}\!\left(
        \theta,\frac{1}{|\mathcal S^+|}\sum_{s\in\mathcal S^+}
        \mathcal L_{\mathrm{CRA}}(s)\right)$
    \EndIf
\EndFor
\State \Return $(\phi^{\star},\theta)$
\end{algorithmic}
\end{algorithm}

\begin{table*}[h]
\centering\footnotesize
\setlength{\tabcolsep}{3.0pt}
\renewcommand{\arraystretch}{1.08}

\begin{tabular*}{\textwidth}{@{\extracolsep{\fill}}llrrrrr@{}}
\toprule
Setting & Student & DG width & $d_c$ & Heads & FFN width &  Intermediate layer \\
\midrule
Distilled students & 0.6B & 128 & 768 & 6 & 3,072 & 19 \\
 & 1.7B & 256 & 1,280 & 10 & 5,120 & 19 \\
 & 4B & 320 & 2,048 & 16 & 8,192 & 24 \\
\midrule
Cross-mode & 0.6B & 128 & 1,024 & 8 & 4,096 & 24 \\
 & 1.7B & 256 & 2,048 & 16 & 8,192 & 24 \\
 & 4B & 320 & 2,048 & 16 & 8,192 & 24 \\
\bottomrule
\end{tabular*}

\caption{SRA module dimensions. DG width is the hidden width of the gate MLP. FFN width is the intermediate dimension of the feedforward network in each CRA Transformer layer.}
\label{tab:sra-architecture}
\end{table*}

\section{Experimental Details}
\label{app:exp-settings}
\label{app:backbone-preparation}

The distillation experiments initialize students from Qwen3 Base and use Qwen3-8B with thinking disabled as the teacher. Tables~\ref{tab:kd-settings} lists the KD, OPD, and EOPD parameters. The OPD prompt and remaining shared settings follow Appendix~\ref{app:opd-checkpoints}. Example~\ref{ex:trajectory} at the end of the appendix shows a complete response of the Qwen3-4B OPD student with SRA on AIME 2025, in which the tokens generated at positions where DG activates are highlighted.

\subsection{SRA Training}
\label{app:sra-exp-training}

SRA uses the training prompt in Appendix~\ref{app:opd-checkpoints} and the procedure in Algorithm~\ref{alg:sra-training}. We initialize DG and CRA separately for each backbone and update only these modules. Table~\ref{tab:sra-settings} summarizes the training hyperparameters.

In the first epoch, the student generates responses on a 90\% split of DAPO-Math-17k. DG learns the TV labels, and CRA is trained at positions with $\mathrm{TV}\geq\tau$. In the second epoch, SRA generates responses on all 17,917 questions. DG is fixed, and CRA is updated at the positions selected by $g_\phi$. During the first epoch, DG makes two optimization passes over each rollout batch, while CRA makes one pass over the selected positions. In the second epoch, DG is frozen and CRA continues with one pass per rollout batch.

DG takes the student's final hidden state as input and predicts whether to activate CRA. It is implemented as a two-layer MLP with dimensions $d\rightarrow d/8\rightarrow 1$, where $d$ is the student hidden size. A ReLU activation follows the first linear layer, and a sigmoid maps the scalar output to a gate score. This score is compared with the threshold $\eta$ to obtain the binary decision $g_\phi(s_t)$.

CRA consists of two bidirectional Transformer layers. At each generation position, its inputs are the student's final hidden state, the hidden state from a selected intermediate layer, and the input embeddings of the student's top-$K$ candidate tokens. Each input type is processed by a separate RMSNorm and linear projection into a common dimension, with normalization and projection parameters shared across candidates. The resulting $K+2$ vectors are jointly processed by multi-head self-attention without a causal mask, allowing each candidate to attend to both hidden state vectors and all candidate vectors. Each Transformer layer contains an attention sublayer and a two-layer feed-forward network with a SiLU activation, both sublayers use pre-RMSNorm and residual connections. Table~\ref{tab:sra-architecture} lists the model widths, attention head counts, and selected intermediate layers.

After the two Transformer layers, CRA retains the $K$ candidate output vectors and applies a shared RMSNorm and scalar linear projection. The resulting scores are centered across candidates to obtain the logit residuals. The scalar projection's weights and bias are initialized to zero, so CRA initially produces zero residuals. 

For Smooth L1 regression, we start from the ideal residual $\delta^\star(s_t)$ defined in \eqref{eq:ideal_residual}, which is already centered over the student's $K$ candidates. We first clip each component to limit extreme target values:
\begin{equation}
\widehat{\delta}_i(s_t)
=\operatorname{clip}\!\left(\delta_i^\star(s_t),-12,12\right).
\end{equation}

Since clipping can change the mean across candidates, we center the clipped values again:
\begin{equation}
\widetilde{\delta}_i(s_t)
=\widehat{\delta}_i(s_t)
-\frac{1}{K}\sum_{j=1}^{K}\widehat{\delta}_j(s_t).
\end{equation}

The resulting target $\widetilde{\delta}(s_t)$ has zero mean, matching the centered residual produced by CRA. Both centering operations are performed across candidates at the same generation position.
Let $r_{t,i}=f_{\theta,i}(s_t)-\widetilde{\delta}_i(s_t)$ denote the prediction error for candidate $i$. We use Smooth L1 with $\beta=1$:
\begin{equation}
\operatorname{SmoothL1}_{\beta=1}(r)=
\begin{cases}
\frac12 r^2, & |r|<1,\\
|r|-\frac12, & |r|\geq1.
\end{cases}
\end{equation}
The residual regression loss averages these errors over the candidate set:
\begin{equation}
\mathcal L_{{s_t}}(\theta)
=\frac{1}{K}\sum_{i=1}^{K}
\operatorname{SmoothL1}_{\beta=1}(r_{t,i}).
\end{equation}

\paragraph{Derivation of the residual target.}
Fix a prefix $s_t$ and omit it for brevity. Let $z_k$ be the student logits on the candidate set, and let $q_S$ and $q_T$ be the student and teacher distributions normalized over this set, with $q_{T,i}>0$. Since sampling uses temperature 1, the student distribution is
\begin{equation}
    q_{S,i}=\exp(z_i)/Z_S , \quad Z_S=\sum_{j=1}^{K}\exp(z_j).
\end{equation}

Let $r_i=\log q_{T,i}-\log q_{S,i}$ and $\bar r=K^{-1}\sum_{j=1}^{K}r_j$, so that the ideal residual in Eq.~\eqref{eq:ideal_residual} is $\delta^\star_i=r_i-\bar r$. Then
\begin{equation}
\begin{aligned}
z_i+\delta^\star_i
&= z_i+\log q_{T,i}-\log q_{S,i}-\bar r\\
&= \log q_{T,i}+\log Z_S-\bar r .
\end{aligned}
\end{equation}
The term $\log Z_S-\bar r$ is the same for all candidates and cancels in the softmax, which gives
\begin{equation}
\operatorname{softmax}\!\left(z_K+\delta^\star\right)=q_T .
\end{equation}
Because the softmax is invariant only to a common shift of its inputs, $\delta^\star$ is the unique residual with zero mean that recovers $q_T$, which matches the centered parameterization $f_\theta=\mathcal C(b_\theta)$. During training, we clip $\delta^\star$ and center it again before using it as the Smooth L1 target.

\subsection{Discrepancy Gate Supervision}
\label{app:gate-supervision-analysis}

\textbf{DG threshold.} After the first epoch, we select $\eta$ by maximizing the gate F1 score~\citep{vanrijsbergen1979information} on a separate validation split of 1,792 DAPO questions using the 4B OPD model. As shown in Figure~\ref{fig:gate-label-f1}, $\eta=0.8$ achieves the highest F1 among the nine tested thresholds. We fix this threshold for all main distillation configurations.

We compare two supervision signals for DG on Qwen3-4B OPD. The TV target is $\mathbf 1[\mathrm{TV}(q_T,q_S)\geq0.1]$. The Top-1 target is $\mathbf 1[\arg\max q_T\neq\arg\max q_S]$, with both distributions normalized on the same student top-16 set. The two runs use the same DG architecture, training settings, and DAPO split, changing only the target labels.

After the first SRA epoch, we evaluate both gates on the reserved 1,792-question validation split and scan gate thresholds from 0.1 to 0.9. Figure~\ref{fig:gate-label-f1} reports each DG's F1 against its corresponding target. Both attain their best F1 at $\eta=0.8$, with 74.08\% for TV and 34.88\% for Top-1.  Direct Top-1 prediction remains challenging for DG under this training setup, while TV supervision yields substantially higher validation F1. 

We attribute this disparity to the representational nature of hidden states: they encapsulate information about the entire next token distribution rather than isolating the discrete top-1 token. Predicting hard top-1 mismatches from dense representations is therefore inherently noisy, whereas TV reflects continuous distributional divergence that aligns naturally with hidden features, making it substantially easier to learn.

\begin{figure}[t]
  \centering
  \includegraphics[width=\columnwidth]{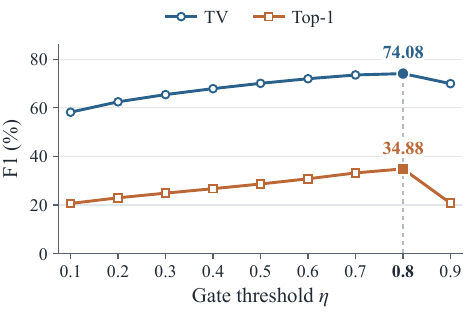}
  \caption{DG F1 on the DAPO validation split after the first SRA epoch with Qwen3-4B OPD. TV and Top-1 denote the two supervision targets. Each curve evaluates a separately trained gate against its own target, with the highest F1 at $\eta=0.8$.}
  \label{fig:gate-label-f1}
\end{figure}

\subsection{Out-of-domain Evaluation}
\label{app:ood-evaluation}
Since our SRA is a nonintervention method, which preserves the capability of the original model better, we speculate that it improves the out-of-domain (OOD) generalization performance of the model \citep{yi2021improved,yi2022breaking,wang2022out}. We evaluate Qwen3-4B students trained for two epochs with KD, OPD, and EOPD, and their corresponding SRA variants, on GPQA-Diamond~\citep{rein2024gpqa}, HumanEval~\citep{chen2021evaluating}, and MMLU-Pro~\citep{wang2024mmlupro}. We generate 16 responses per question with temperature=1, top $p=1$, and a maximum generation length of 31,744 tokens. The prompts used for each benchmark are given below.

\paragraph{HumanEval.}
We use the Qwen3 chat template with non-thinking mode. The system message requests Python code, and the user message is the original HumanEval prompt containing the function signature, docstring, and public examples. 

\begin{sraevaluationprompt}
\ttfamily\raggedright
\textnormal{\bfseries System}\par
You write Python code. Output only Python code. No markdown, no explanation.
\par\medskip
\textnormal{\bfseries User}\par
\{original HumanEval prompt\}
\end{sraevaluationprompt}

\paragraph{MMLU-Pro.}
We use a plain-text prompt with five CoT demonstrations from the validation set of the same subject. Each demonstration contains a question, its options, a rationale, and the answer. Rationales and answers appear only in the demonstrations. The test item ends with the reasoning cue shown below. The five demonstration slots in the template each expand to one complete validation example.

\begin{sraevaluationprompt}
\ttfamily\raggedright
The following are multiple choice questions (with answers) about \{subject\}. Think step by step and then finish your answer with "the answer is (X)" where X is the correct letter choice.
\par\medskip
\{validation\_example\_1\}
\par\medskip
\{validation\_example\_2\}
\par\medskip
\{validation\_example\_3\}
\par\medskip
\{validation\_example\_4\}
\par\medskip
\{validation\_example\_5\}
\par\medskip
Question:\par
\{test\_question\}\par
Options:\par
A. \{option\_A\}\par
B. \{option\_B\}\par
...\par
Answer: Let's think step by step.
\end{sraevaluationprompt}

\paragraph{GPQA-Diamond.}
We use a zero-shot user prompt with the Qwen3 chat template. Each question is followed by its four options and the instruction below. 

\begin{sraevaluationprompt}
\ttfamily\raggedright
\textnormal{\bfseries User}\par
\{question\}
\par\medskip
A. \{option\_A\}\par
B. \{option\_B\}\par
C. \{option\_C\}\par
D. \{option\_D\}
\par\medskip
Please reason step by step, and put your final answer within \textbackslash boxed\{\}.
\end{sraevaluationprompt}

Table~\ref{tab:ood-results} shows that SRA consistently enhances student generalization across diverse tasks without inducing catastrophic forgetting. Across all distillation baselines, SRA yields positive gains in both Avg@16 and Pass@16, uniformly boosting scientific reasoning on GPQA-Diamond while fully preserving coding proficiency on HumanEval and multidisciplinary accuracy on MMLU-Pro. These results confirm that SRA safely strengthens the policy without degrading broader capabilities.

\begin{table}[t]
\centering\begingroup
\fontsize{9}{9.5}\selectfont
\setlength{\tabcolsep}{0.8pt}
\renewcommand{\arraystretch}{1.0}
\setlength{\aboverulesep}{1.4pt}
\setlength{\belowrulesep}{1.8pt}
\setlength{\dimen0}{\dimexpr\columnwidth-10\tabcolsep\relax}
\begin{tabular}{@{}>{\raggedright\arraybackslash\leavevmode}p{.204\dimen0}>{\raggedright\arraybackslash\leavevmode}p{.133\dimen0}>{\raggedleft\arraybackslash\leavevmode}p{.153\dimen0}>{\raggedleft\arraybackslash\leavevmode}p{.158\dimen0}>{\raggedleft\arraybackslash\leavevmode}p{.183\dimen0}>{\raggedleft\arraybackslash\leavevmode}p{.169\dimen0}@{}}
\toprule
\multirow{2}{*}{Benchmark} & \multirow{2}{*}{Method} & \multicolumn{2}{c}{Student} & \multicolumn{2}{c}{\cellcolor{black!6}w/~SRA} \\
\cmidrule(lr){3-4}\cmidrule(l){5-6}
 & & Avg@16 & Pass@16 & \cellcolor{black!6}Avg@16 & \cellcolor{black!6}Pass@16 \\
\midrule
\multirow{3}{*}{\shortstack[l]{GPQA-\\Diamond}} & KD & 31.60 & 90.40 & \cellcolor{black!6}32.70 & \cellcolor{black!6}\textbf{91.41} \\
 & OPD & 31.31 & 89.39 & \cellcolor{black!6}\textbf{33.68} & \cellcolor{black!6}\underline{90.91} \\
 & EOPD & 32.23 & 87.88 & \cellcolor{black!6}\underline{32.95} & \cellcolor{black!6}90.40 \\
\midrule
\multirow{3}{*}{HumanEval} & KD & \underline{75.91} & 94.51 & \cellcolor{black!6}\textbf{76.49} & \cellcolor{black!6}94.51 \\
 & OPD & 71.34 & \textbf{95.73} & \cellcolor{black!6}71.65 & \cellcolor{black!6}\textbf{95.73} \\
 & EOPD & 71.80 & \underline{95.12} & \cellcolor{black!6}72.03 & \cellcolor{black!6}\textbf{95.73} \\
\midrule
\multirow{3}{*}{MMLU-Pro} & KD & 48.54 & \underline{89.60} & \cellcolor{black!6}49.37 & \cellcolor{black!6}\textbf{89.99} \\
 & OPD & 51.13 & 88.43 & \cellcolor{black!6}51.81 & \cellcolor{black!6}87.84 \\
 & EOPD & \underline{52.69} & 87.56 & \cellcolor{black!6}\textbf{52.82} & \cellcolor{black!6}87.93 \\
\midrule
\multirow{3}{*}{Average} & KD & 52.02 & \underline{91.50} & \cellcolor{black!6}\textbf{52.85}\textsuperscript{\textcolor{gain}{+0.83}} & \cellcolor{black!6}\textbf{91.97}\textsuperscript{\textcolor{gain}{+0.47}} \\
 & OPD & 51.26 & 91.18 & \cellcolor{black!6}52.38\textsuperscript{\textcolor{gain}{+1.12}} & \cellcolor{black!6}91.49\textsuperscript{\textcolor{gain}{+0.31}} \\
 & EOPD & 52.24 & 90.19 & \cellcolor{black!6}\underline{52.60}\textsuperscript{\textcolor{gain}{+0.36}} & \cellcolor{black!6}91.35\textsuperscript{\textcolor{gain}{+1.16}} \\
\bottomrule
\end{tabular}\endgroup
\caption{Out-of-domain results of Qwen3-4B distilled students. Shaded columns use SRA.}
\label{tab:ood-results}
\end{table}

\begin{figure}[t]
  \centering
  \includegraphics[width=\columnwidth]{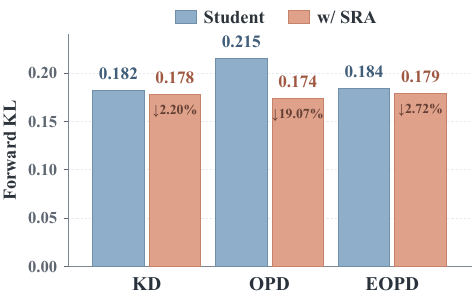}
  \caption{Forward KL within student top-16 candidates before and after SRA correction on Qwen3-4B. Reported values average DG selected positions across four mathematical benchmarks.}
  \label{fig:sra-fkl-alignment}
\end{figure}

\subsection{Teacher Alignment}
\label{app:teacher-alignment}

To assess teacher alignment, we analyze Qwen3-4B students across KD, OPD, and EOPD based on their SRA-generated responses on AIME24, AIME25, AMC23, and MATH-500. Specifically, at prefixes selected by the Discrepancy Gate, we compare the forward KL divergence between the teacher and student, $D_{\mathrm{KL}}(q_T\|q_S)$, against that of the corrected distribution, $D_{\mathrm{KL}}(q_T\|q_\theta)$. Both distributions are normalized over the shared student top-16 candidates. We compute averages across selected positions within each benchmark and report the macro average across all four tasks.

As depicted in Figure~\ref{fig:sra-fkl-alignment}, SRA consistently curtails forward KL divergence across all three distillation baselines. This reduction is most pronounced on OPD, dropping from 0.215 to 0.174, a relative improvement of 19.07\%. These results confirm that SRA reshapes probability allocation over student candidates, driving local predictions substantially closer to teacher behavior.

\begin{table*}[h]
\centering\begingroup
\fontsize{9}{9.5}\selectfont
\setlength{\tabcolsep}{2pt}
\renewcommand{\arraystretch}{1.0}
\setlength{\dimen0}{\dimexpr\textwidth-12\tabcolsep\relax}
\begin{tabular}{@{}>{\raggedright\arraybackslash\leavevmode}p{.22\dimen0}*{4}{>{\raggedleft\arraybackslash\leavevmode}p{.15\dimen0}}*{2}{>{\raggedleft\arraybackslash\leavevmode}p{.09\dimen0}}@{}}
\toprule
 & \multicolumn{4}{c}{Avg@16 / Pass@16 (\%)} & \multicolumn{2}{c}{Average (\%)} \\
\cmidrule(lr){2-5}\cmidrule(l){6-7}
Setting & AIME24 & AIME25 & AMC23 & MATH-500 & Avg@16 & Pass@16 \\
\midrule
\multicolumn{7}{l}{\textit{Final hidden state}} \\
w/o final hidden state & \underline{15.21}\,/\,\textbf{40.00} & \textbf{15.00}\,/\,\underline{43.33} & \underline{55.16}\,/\,\underline{87.50} & \underline{78.14}\,/\,\underline{93.80} & \underline{40.88} & \underline{66.16} \\
\rowcolor{black!6}
SRA & \textbf{16.04}\,/\,\textbf{40.00} & \underline{14.38}\,/\,\textbf{46.67} & \textbf{56.88}\,/\,\textbf{90.00} & \textbf{78.49}\,/\,\textbf{95.00} & \textbf{41.45} & \textbf{67.92} \\
\midrule
\multicolumn{7}{l}{\textit{Intermediate hidden state}} \\
w/o intermediate state & 13.33\,/\,30.00 & \textbf{16.04}\,/\,\underline{43.33} & 54.69\,/\,85.00 & 77.59\,/\,\underline{94.80} & 40.41 & 63.28 \\
L12 & \textbf{16.04}\,/\,\textbf{40.00} & \underline{15.42}\,/\,40.00 & 54.84\,/\,85.00 & \underline{77.84}\,/\,93.40 & 41.03 & 64.60 \\
\rowcolor{black!6}
L24 (SRA) & \textbf{16.04}\,/\,\textbf{40.00} & 14.38\,/\,\textbf{46.67} & \underline{56.88}\,/\,\textbf{90.00} & \textbf{78.49}\,/\,\textbf{95.00} & \textbf{41.45} & \textbf{67.92} \\
L32 & \underline{15.62}\,/\,\underline{36.67} & 14.38\,/\,\textbf{46.67} & \textbf{57.19}\,/\,\underline{87.50} & 77.81\,/\,94.00 & \underline{41.25} & \underline{66.21} \\
\midrule
\multicolumn{7}{l}{\textit{Residual supervision}} \\
$\lambda_R=0$ & 15.83\,/\,\underline{40.00} & \textbf{15.00}\,/\,\underline{40.00} & 53.91\,/\,\textbf{90.00} & 77.39\,/\,\underline{94.20} & 40.53 & \underline{66.05} \\
\rowcolor{black!6}
$\lambda_R=0.1$ (SRA) & \underline{16.04}\,/\,\underline{40.00} & \underline{14.38}\,/\,\textbf{46.67} & \textbf{56.88}\,/\,\textbf{90.00} & \textbf{78.49}\,/\,\textbf{95.00} & \textbf{41.45} & \textbf{67.92} \\
$\lambda_R=1$ & \textbf{16.88}\,/\,\textbf{43.33} & 13.54\,/\,36.67 & \underline{55.62}\,/\,\underline{87.50} & \underline{78.15}\,/\,93.60 & \underline{41.05} & 65.28 \\
\midrule
\multicolumn{7}{l}{\textit{Online training}} \\
Offline training & \underline{15.63}\,/\,\underline{33.33} & \textbf{17.29}\,/\,\underline{36.67} & \underline{54.38}\,/\,\underline{87.50} & \underline{78.15}\,/\,\underline{93.80} & \underline{41.36} & \underline{62.83} \\
\rowcolor{black!6}
SRA & \textbf{16.04}\,/\,\textbf{40.00} & \underline{14.38}\,/\,\textbf{46.67} & \textbf{56.88}\,/\,\textbf{90.00} & \textbf{78.49}\,/\,\textbf{95.00} & \textbf{41.45} & \textbf{67.92} \\
\midrule
\multicolumn{7}{l}{\textit{Alternative adaptations}} \\
OPD (2 epochs) & \underline{14.79}\,/\,36.67 & \textbf{15.83}\,/\,33.33 & 54.22\,/\,\textbf{90.00} & \underline{78.00}\,/\,\underline{94.20} & \underline{40.71} & 63.55 \\
Extra decoder & 14.37\,/\,30.00 & 13.96\,/\,33.33 & 53.75\,/\,\underline{87.50} & 77.70\,/\,93.00 & 39.95 & 60.96 \\
CRA without DG & 14.58\,/\,\textbf{43.33} & \underline{14.58}\,/\,\underline{40.00} & \underline{54.37}\,/\,82.50 & 77.58\,/\,93.20 & 40.28 & \underline{64.76} \\
\rowcolor{black!6}
SRA & \textbf{16.04}\,/\,\underline{40.00} & 14.38\,/\,\textbf{46.67} & \textbf{56.88}\,/\,\textbf{90.00} & \textbf{78.49}\,/\,\textbf{95.00} & \textbf{41.45} & \textbf{67.92} \\
\midrule
\multicolumn{7}{l}{\textit{Candidate set size}} \\
$K=4,\ \eta=0.6$ & \textbf{16.25}\,/\,\underline{36.67} & \textbf{15.00}\,/\,\textbf{50.00} & 53.12\,/\,\underline{87.50} & 77.89\,/\,\underline{94.20} & 40.57 & \underline{67.09} \\
$K=8,\ \eta=0.7$ & \textbf{16.25}\,/\,\textbf{40.00} & 13.96\,/\,33.33 & \underline{55.62}\,/\,\underline{87.50} & 78.12\,/\,93.20 & \underline{40.99} & 63.51 \\
\rowcolor{black!6}
$K=16,\ \eta=0.8$ (SRA) & \underline{16.04}\,/\,\textbf{40.00} & \underline{14.38}\,/\,\underline{46.67} & \textbf{56.88}\,/\,\textbf{90.00} & \textbf{78.49}\,/\,\textbf{95.00} & \textbf{41.45} & \textbf{67.92} \\
$K=32,\ \eta=0.7$ & 15.83\,/\,\underline{36.67} & 14.37\,/\,36.67 & 55.00\,/\,\underline{87.50} & \underline{78.17}\,/\,94.00 & 40.85 & 63.71 \\
\bottomrule
\end{tabular}\endgroup
\caption{Complete ablations and adaptation comparisons on Qwen3-4B OPD. Shaded rows repeat the same SRA reference. Bold and underline indicate the best and second best scores within each group. For candidate size ablations, $\eta$ is the DG threshold that achieves the highest F1 score on the validation set.}
\label{tab:ablation-overview}
\label{tab:component-tasks}
\label{tab:additional-comparisons}
\label{tab:ablation-tasks}
\label{tab:candidate-size-ablation}
\end{table*}

\section{Additional Comparisons and Ablation Details}
\label{app:full-results}
\label{app:ablation-settings}

This appendix provides the complete benchmark results for the component ablations in Section~\ref{sec:component-ablations}, together with additional comparisons of offline and online training, alternative adaptation architectures, and candidate set size $K$. All experiments use a frozen Qwen3-4B OPD student trained for two epochs, with results summarized in Table~\ref{tab:ablation-overview}. The full SRA reference is the configuration reported in Table~\ref{tab:main-results}, using L24 and $\lambda_R=0.1$.

\subsection{Component Ablations}
\label{app:component-ablations}

Unless otherwise specified, these experiments use the SRA training settings in Appendix~\ref{app:sra-exp-training}.

\paragraph{Final hidden state.}
We remove the final hidden state from CRA while retaining the intermediate hidden state from L24 and the candidate token embeddings. This comparison examines whether the final representation provides information beyond the other two inputs.

\paragraph{Intermediate hidden state.}
We first remove the intermediate hidden state to examine its contribution. We then compare intermediate layers 12, 24, and 32 while retaining the final hidden state and candidate embeddings. The default uses L24.

\paragraph{Residual supervision.}
We vary the Smooth L1 coefficient $\lambda_R$ over $\{0,0.1,1\}$, keeping the forward KL coefficient fixed at 1. Setting $\lambda_R=0$ removes direct residual regression, the remaining values compare two regression strengths. The default uses $\lambda_R=0.1$.

\paragraph{Results.}
Table~\ref{tab:ablation-overview} presents the component ablation results on Qwen3-4B, demonstrating that each architectural design in SRA is necessary and mutually reinforcing.
Removing the final hidden state incurs drops of 0.57 points and 1.76 points on average Avg@16 and Pass@16, respectively, confirming that sequence context from the final layer provides indispensable guidance beyond candidate token embeddings.
Omitting the intermediate hidden state triggers the most severe degradation, slashing average Avg@16 by 1.04 points and Pass@16 by 4.64 points, with AIME24 Pass@16 collapsing from 40.00\% to 30.00\%. Across representation depths, L24 consistently outperforms both L12 and L32 on all metrics, proving that mid-level representations offer the optimal depth for stabilizing reasoning trajectories.
Eliminating residual supervision under $\lambda_R=0$ decreases Avg@16 and Pass@16 by 0.92 points and 1.87 points, while overly aggressive regression under $\lambda_R=1$ also harms sampling coverage, dropping Pass@16 to 65.28\%. The default setting of $\lambda_R=0.1$ achieves the highest overall accuracy, validating that moderate residual guidance complements distribution matching.
These consistent performance drops across ablated variants verify the efficacy of our proposed design.

\subsection{Additional Ablation Studies}
\label{app:additional-comparisons}

We further examine candidate set size, training strategy, and module architecture. All comparisons use the same frozen Qwen3-4B OPD student trained for two epochs, with the additional modules trained for two epochs. Table~\ref{tab:ablation-overview} summarizes the results.

\paragraph{Candidate set size.}
We evaluate $K \in \{4, 8, 16, 32\}$ and train DG and CRA separately for each configuration while keeping the CRA architecture, intermediate layer L24 input, learning rate, and loss coefficients identical. After the first epoch, we select the DG threshold maximizing the validation F1 score. Among all choices, $K=16$ yields the best overall performance, attaining 41.45\% average Avg@16 and 67.92\% Pass@16. Expanding the candidate pool to $K=32$ degrades both metrics to 40.85\% and 63.71\%, confirming $K=16$ as the optimal configuration.

\paragraph{Training strategy.}
We compare the standard SRA procedure against an offline variant that reuses 17,917 static trajectories generated once by the frozen student across both stages. After stage one, DG is frozen and CRA continues training on the same fixed prefixes without refreshing rollouts. While the offline variant achieves a comparable average Avg@16 of 41.36\%, its average Pass@16 drops sharply to 62.83\% compared to 67.92\% for full SRA. This gap highlights the necessity of training CRA on dynamic rollouts generated by the enhanced policy.

\paragraph{Module architecture.}
We evaluate two alternatives to the complete DG and CRA design.
For the extra decoder baseline, we insert a duplicate of the student's final decoder layer before the final normalization, initialize its attention and MLP output projections to zero, and train only this added layer with 100.93M parameters, compared to 117.27M for DG and CRA together. It uses causal attention over the full vocabulary and trains with the OPD objective using AdamW, a learning rate of $3\times10^{-6}$, and a cosine schedule.

For CRA without DG, we retain CRA's architecture and loss while removing DG and the TV selection mask. CRA trains on all valid response positions and applies residual corrections across every generation step. Training runs for two epochs with AdamW using a constant learning rate of $5\times10^{-5}$, following the batch sizes and sampling setup in Appendix~\ref{app:sra-exp-training}. During evaluation, every step samples from the corrected top-16 distribution.

As reported in Table~\ref{tab:ablation-overview}, full SRA substantially outperforms both alternatives on Avg@16 and Pass@16, where the extra decoder and CRA without DG only achieve 60.96\% and 64.76\% Pass@16, respectively. These results confirm that selective residual gating provides far superior policy boosting over adding native decoder depth or unconstrained global correction.

\begin{table*}[ht]
\centering\footnotesize
\setlength{\tabcolsep}{2pt}
\renewcommand{\arraystretch}{1.08}
\setlength{\dimen0}{\dimexpr\textwidth-12\tabcolsep\relax}
\begin{tabular}{@{}>{\raggedright\arraybackslash}p{.20\dimen0}*{4}{>{\raggedleft\arraybackslash}p{.13\dimen0}}*{2}{>{\raggedleft\arraybackslash}p{.14\dimen0}}@{}}
\toprule
 & \multicolumn{6}{c}{Avg@16 / Pass@16 (\%)} \\
\cmidrule(l){2-7}
Setting & AIME24 & AIME25 & AMC23 & MATH-500 & Average & Recovery \\
\midrule
\multicolumn{7}{c}{\textit{Qwen3-0.6B}} \\
\midrule
non-thinking & \makebox[23pt][r]{1.25}\,/\,\makebox[25pt][r]{3.33} & \makebox[23pt][r]{2.50}\,/\,\makebox[25pt][r]{16.67} & \makebox[23pt][r]{23.75}\,/\,\makebox[25pt][r]{62.50} & \makebox[23pt][r]{43.16}\,/\,\makebox[25pt][r]{81.60} & \makebox[23pt][r]{17.67}\,/\,\makebox[25pt][r]{41.02} & -- \\
\rowcolor{black!6}
\quad w/~SRA & \makebox[23pt][r]{\underline{3.33}}\,/\,\makebox[25pt][r]{\underline{13.33}} & \makebox[23pt][r]{\underline{6.04}}\,/\,\makebox[25pt][r]{\underline{23.33}} & \makebox[23pt][r]{\underline{29.22}}\,/\,\makebox[25pt][r]{\underline{72.50}} & \makebox[23pt][r]{\underline{53.74}}\,/\,\makebox[25pt][r]{\underline{88.60}} & \makebox[23pt][r]{\underline{23.08}}\,/\,\makebox[25pt][r]{\underline{49.44}} & \makebox[23pt][r]{62.94}\,/\,\makebox[25pt][r]{79.05} \\
thinking & \makebox[23pt][r]{\textbf{11.67}}\,/\,\makebox[25pt][r]{\textbf{43.33}} & \makebox[23pt][r]{\textbf{14.58}}\,/\,\makebox[25pt][r]{\textbf{33.33}} & \makebox[23pt][r]{\textbf{48.12}}\,/\,\makebox[25pt][r]{\textbf{82.50}} & \makebox[23pt][r]{\textbf{72.29}}\,/\,\makebox[25pt][r]{\textbf{91.00}} & \makebox[23pt][r]{\textbf{36.67}}\,/\,\makebox[25pt][r]{\textbf{62.54}} & -- \\
\midrule
\multicolumn{7}{c}{\textit{Qwen3-1.7B}} \\
\midrule
non-thinking & \makebox[23pt][r]{13.33}\,/\,\makebox[25pt][r]{43.33} & \makebox[23pt][r]{10.62}\,/\,\makebox[25pt][r]{33.33} & \makebox[23pt][r]{44.06}\,/\,\makebox[25pt][r]{82.50} & \makebox[23pt][r]{71.00}\,/\,\makebox[25pt][r]{91.80} & \makebox[23pt][r]{34.76}\,/\,\makebox[25pt][r]{62.74} & -- \\
\rowcolor{black!6}
\quad w/~SRA & \makebox[23pt][r]{\underline{39.58}}\,/\,\makebox[25pt][r]{\underline{73.33}} & \makebox[23pt][r]{\underline{30.63}}\,/\,\makebox[25pt][r]{\underline{56.67}} & \makebox[23pt][r]{\underline{72.66}}\,/\,\makebox[25pt][r]{\underline{95.00}} & \makebox[23pt][r]{\underline{86.06}}\,/\,\makebox[25pt][r]{\textbf{97.00}} & \makebox[23pt][r]{\underline{57.23}}\,/\,\makebox[25pt][r]{\underline{80.50}} & \makebox[23pt][r]{85.74}\,/\,\makebox[25pt][r]{93.58} \\
thinking & \makebox[23pt][r]{\textbf{51.67}}\,/\,\makebox[25pt][r]{\textbf{80.00}} & \makebox[23pt][r]{\textbf{38.54}}\,/\,\makebox[25pt][r]{\textbf{70.00}} & \makebox[23pt][r]{\textbf{87.03}}\,/\,\makebox[25pt][r]{\textbf{97.50}} & \makebox[23pt][r]{\textbf{89.76}}\,/\,\makebox[25pt][r]{\underline{96.60}} & \makebox[23pt][r]{\textbf{66.75}}\,/\,\makebox[25pt][r]{\textbf{86.02}} & -- \\
\midrule
\multicolumn{7}{c}{\textit{Qwen3-4B}} \\
\midrule
non-thinking & \makebox[23pt][r]{24.58}\,/\,\makebox[25pt][r]{\underline{50.00}} & \makebox[23pt][r]{20.42}\,/\,\makebox[25pt][r]{\underline{46.67}} & \makebox[23pt][r]{67.66}\,/\,\makebox[25pt][r]{\underline{95.00}} & \makebox[23pt][r]{82.75}\,/\,\makebox[25pt][r]{96.20} & \makebox[23pt][r]{48.85}\,/\,\makebox[25pt][r]{71.97} & -- \\
\rowcolor{black!6}
\quad w/~SRA & \makebox[23pt][r]{\underline{68.96}}\,/\,\makebox[25pt][r]{\textbf{83.33}} & \makebox[23pt][r]{\underline{61.88}}\,/\,\makebox[25pt][r]{\textbf{83.33}} & \makebox[23pt][r]{\underline{95.16}}\,/\,\makebox[25pt][r]{\textbf{100.00}} & \makebox[23pt][r]{\underline{93.21}}\,/\,\makebox[25pt][r]{\textbf{98.00}} & \makebox[23pt][r]{\underline{79.80}}\,/\,\makebox[25pt][r]{\textbf{91.17}} & \makebox[23pt][r]{96.38}\,/\,\makebox[25pt][r]{100.22} \\
thinking & \makebox[23pt][r]{\textbf{72.29}}\,/\,\makebox[25pt][r]{\textbf{83.33}} & \makebox[23pt][r]{\textbf{67.08}}\,/\,\makebox[25pt][r]{\textbf{83.33}} & \makebox[23pt][r]{\textbf{97.66}}\,/\,\makebox[25pt][r]{\textbf{100.00}} & \makebox[23pt][r]{\textbf{94.17}}\,/\,\makebox[25pt][r]{\underline{97.20}} & \makebox[23pt][r]{\textbf{82.80}}\,/\,\makebox[25pt][r]{\underline{90.97}} & -- \\
\bottomrule
\end{tabular}
\caption{Cross-mode results for Qwen3. Shaded rows apply SRA in non-thinking mode. Bold indicates best performance and underline indicates second best within each size. Recovery is the fraction of thinking performance attained by SRA.}
\label{tab:cross-mode}
\end{table*}

\begin{table*}[h]
\centering\begingroup
\fontsize{9}{10}\selectfont
\setlength{\tabcolsep}{3pt}
\renewcommand{\arraystretch}{1.08}
\setlength{\dimen0}{\dimexpr\textwidth-10\tabcolsep\relax}
\begin{tabular}{@{}>{\raggedright\arraybackslash}p{0.19\dimen0}>{\raggedleft\arraybackslash}p{0.162\dimen0}>{\raggedleft\arraybackslash}p{0.162\dimen0}>{\raggedleft\arraybackslash}p{0.162\dimen0}>{\raggedleft\arraybackslash}p{0.162\dimen0}>{\raggedleft\arraybackslash}p{0.162\dimen0}@{}}
\toprule
 & \multicolumn{5}{c}{Avg@16 / Pass@16 (\%)} \\
\cmidrule(l){2-6}
Method & AIME24 & AIME25 & AMC23 & MATH-500 & Average \\
\midrule
Initial Instruct (3B) & \underline{1.25}\,/\, \underline{6.67} & 0.00\,/\, 0.00 & 9.22\,/\, \textbf{55.00} & 27.91\,/\, \underline{71.80} & 9.60\,/\, \underline{33.37} \\
\rowcolor{black!6}
\quad w/~SRA & \textbf{1.67}\,/\, \textbf{13.33} & \underline{0.21}\,/\, \underline{3.33} & \textbf{10.78}\,/\, \textbf{55.00} & \underline{28.31}\,/\, \textbf{73.40} & \underline{10.24}\,/\, \textbf{36.27} \\
\midrule
KD & 0.42\,/\, \underline{6.67} & 0.00\,/\, 0.00 & 8.12\,/\, 47.50 & 24.40\,/\, 63.60 & 8.24\,/\, 29.44 \\
OPD & 0.42\,/\, \underline{6.67} & \underline{0.21}\,/\, \underline{3.33} & \underline{10.00}\,/\, \underline{50.00} & \textbf{32.10}\,/\, 69.60 & \textbf{10.68}\,/\, 32.40 \\
EOPD & 0.00\,/\, 0.00 & \textbf{0.42}\,/\, \textbf{6.67} & 8.91\,/\, 42.50 & 27.09\,/\, 65.00 & 9.10\,/\, 28.54 \\
\midrule
Teacher (8B) & 1.25\,/\, 16.67 & 0.00\,/\, 0.00 & 14.84\,/\, 57.50 & 35.11\,/\, 74.80 & 12.80\,/\, 37.24 \\
\bottomrule
\end{tabular}
\endgroup
\caption{Performance comparison of Llama-3.2-3B-Instruct and distillation baselines with an 8B teacher reference. The shaded row applies SRA to the initial Instruct model. Bold indicates best performance and underline indicates second best among students.}
\label{tab:llama-results}
\end{table*}

\section{Cross-Mode Distillation}
\label{app:cross-mode}

\paragraph{Experimental setting.}
We further evaluate whether SRA can transfer reasoning capabilities between the two modes of the same model. At each scale, we treat the non-thinking mode of the  Qwen3 model as the student and its thinking mode as the teacher. The teacher provides probability targets along the student responses, while SRA is trained with the model parameters frozen. Related work also studies distillation between the thinking and non-thinking modes of a single model~\citep{kalwar2026masked}.

Training follows Appendix~\ref{app:sra-exp-training}, with the module dimensions and intermediate layers specified in Table~\ref{tab:sra-architecture}. After the first epoch, the gate threshold is selected by maximizing F1~\citep{vanrijsbergen1979information} on a separate validation split. All three scales select $\eta=0.7$. Evaluation follows Section~\ref{sec:experiments}.

For each metric $m\in\{\mathrm{Avg@16},\mathrm{Pass@16}\}$, we report the fraction of native thinking performance attained by SRA:
\begin{equation}
\mathrm{Recovery}_m=
\frac{\bar m_{\mathrm{SRA}}}{\bar m_{\mathrm{thinking}}}\times100\%,
\label{eq:thinking-recovery}
\end{equation}
where $\bar m$ denotes the Average score across the four benchmarks.

\paragraph{Results.}
Table~\ref{tab:cross-mode} shows that SRA improves both average Avg@16 and Pass@16 over the native non-thinking mode at every model size. It attains 62.94\%, 85.74\%, and 96.38\% of native thinking Avg@16 at 0.6B, 1.7B, and 4B, respectively. At 4B, Avg@16 increases from 48.85\% to 79.80\%, approaching the thinking mode's 82.80\%, while Pass@16 reaches 91.17\%, compared with 90.97\% for native thinking. These results show that SRA can use supervision from a model's own thinking mode to improve its non-thinking reasoning performance.

\section{Experiments on LLaMA}
\label{app:LLaMA-results}

We use Llama-3.2-3B-Instruct~\citep{meta2024llama32} as the initial student and Llama-3.1-8B-Instruct as the teacher. KD, OPD, and EOPD train for two epochs. SRA trains DG and CRA directly on the frozen initial Instruct model.

\paragraph{Training.} We use DAPO-Math-17k and the distillation objectives in Appendix~\ref{app:backbone-preparation}. SRA follows Appendix~\ref{app:sra-exp-training}, with $K=16$, TV threshold $\tau=0.1$, and gate threshold $\eta=0.9$ fixed during the second training epoch and evaluation. DG and CRA use learning rates of $10^{-4}$ and $5\times10^{-5}$. CRA uses the layer-19 representation and two Transformer layers with width 2,048, 16 attention heads, and feedforward width 8,192.

\paragraph{Evaluation.} All models use temperature=1, top-$p=1$, 16 responses per question, and a response limit of 31,744 tokens. We use the native Llama chat template with the evaluation prompt in Appendix~\ref{app:prompts}.

\paragraph{Results.} Table~\ref{tab:llama-results} presents the evaluation results on mathematical reasoning benchmarks. Applying SRA consistently enhances the initial Instruct baseline, elevating the average Avg@16 / Pass@16 from 9.60 / 33.37 to 10.24 / 36.27 and achieving the highest Pass@16 score among all student variants, nearly matching the 8B teacher performance of 37.24\%. In contrast, standard distillation baselines perform suboptimally. KD and EOPD degrade both metrics, and all three distillation methods reduce Pass@16 relative to the initial model. Although OPD yields a slightly higher Avg@16, its drop in Pass@16 aligns with prior observations that distillation can degrade post-trained student models~\citep{li2026rethinking}.

\begin{figure*}[p]\refstepcounter{sraexample}\label{ex:trajectory}
\begin{sraexbox}{Example~\thesraexample: AIME 2025 problem under SRA (1/3)}
\textbf{Question.} In $\triangle ABC$ points $D$ and $E$ lie on $\overline{AB}$ so that $AD < AE < AB$, while points $F$ and $G$ lie on $\overline{AC}$ so that $AF < AG < AC$. Suppose $AD = 4$, $DE = 16$, $EB = 8$, $AF = 13$, $FG = 52$, and $GC = 26$. Let $M$ be the reflection of $D$ through $F$, and let $N$ be the reflection of $G$ through $E$. The area of quadrilateral $DEGF$ is $288$. Find the area of heptagon $AFNBCEM$.

\smallskip\textbf{Reference answer.} $588$

\sraexrule
{\footnotesize \gk{Blue}: DG activates and the sampled token is the student's top-1 candidate. \gc{Orange}: DG activates and the sampled token is not the student's top-1 candidate. Unmarked tokens are sampled from $\pi_S$ with $g_\phi(s_t)=0$. A highlighted line break appears as a short block at the end of a line. DG activates at 269 of the 2,510 tokens, 10.7\%, and 111 of these positions sample a token other than the student's top-1 candidate.\par}
\sraexrule
{\ttfamily\small\raggedright\setlength{\parindent}{0pt}\setlength{\parskip}{0pt}
We are given \gk{a} \gk{triangle} \$ \textbackslash{}triangle ABC \$ with \gk{points}\gc{:}\par
\vspace{0.3\baselineskip}
{-} \$ D \$ and \$ E \$ on \$ \textbackslash{}overline\{AB\} \$, where \$ AD =  4 \$, \$ DE =  16 \$, \$ EB =  8 \$,\par
{-} \$ F \$ and \$ G \$ on \$ \textbackslash{}overline\{AC\} \$, where \$ AF =  13 \$, \$ FG =  52 \$, \$ GC =  26 \$\gk{,}\par
{-} \gk{\$} M \$ is the reflection of \$ D \$ through \$ F \$,\par
{-} \$ N \$ is the reflection of \$ G \$ through \$ E \$,\par
{-} \gk{The} area of quadrilateral \$ DEGF \$ is  288\gk{.}\par
\vspace{0.3\baselineskip}
We are to find the area of heptagon \$ AFNBCEM \$.\par
\vspace{0.3\baselineskip}
{-}{-}{-}\par
\vspace{0.3\baselineskip}
\#\#\# Step  1: \gk{Determine} \gk{the} \gc{length} of \gk{\$} AB \gk{\$} \gk{and} \$ AC \$\par
\vspace{0.3\baselineskip}
\gk{From} \gk{the} \gc{segment} \gc{data}\gc{:}\par
\vspace{0.3\baselineskip}
{-} \$ AD =  4 \$, \$ DE =  16 \$, \$ EB =  8 \gk{\$}, so\gc{\hspace{0.5em}}\par
\gk{~} \$\$\par
~ ~AB = AD + DE + EB =  4 +  16 +  8 =  28\par
~ ~\gk{\$\$}\par
\vspace{0.3\baselineskip}
{-} \$ AF =  13 \$, \$ FG =  52 \$, \$ GC =  26 \$, so\par
~ ~\$\$\par
~ ~AC = AF + FG + GC =  13 +  52 +  26 =  91\par
~ ~\$\$\par
\vspace{0.3\baselineskip}
{-}{-}{-}\par
\vspace{0.3\baselineskip}
\#\#\# Step  2: \gc{Understand} \gk{the} \gc{reflections}\gk{\hspace{0.5em}}\par
\vspace{0.3\baselineskip}
\gk{{-}} \$ M \$ is the reflection of \$ D \$ through \$ F \$. \gc{That} means \gk{\$} F \$ is the midpoint of \$ \gk{DM} \$, so\gc{\hspace{0.5em}}\par
~ ~\$\$\par
~ ~\gc{FM} = FD\gk{\hspace{0.5em}}\par
~ ~\gk{\$\$}\par
\vspace{0.3\baselineskip}
{-} \$ N \$ is the reflection of \$ G \$ through \$ E \$. That means \$ E \$ is the midpoint of \$ GN \$, so\par
~ ~\$\$\par
~ ~\gk{EN} = EG\par
~ ~\$\$\par
\vspace{0.3\baselineskip}
\gk{{-}{-}{-}}\par
\vspace{0.3\baselineskip}
\#\#\# Step  3: \gk{Use} \gc{coordinate} geometry \gk{to} \gc{place} \gc{the} points\par
\vspace{0.3\baselineskip}
\gk{Let}\textquotesingle{}s \gk{place} \gc{the} triangle \gc{\$} ABC \$ in the coordinate plane\gc{.}\par
\vspace{0.3\baselineskip}
\gk{Let}\textquotesingle{}s \gc{set}\gk{:}\par
\vspace{0.3\baselineskip}
{-} \$ A = (0,  0) \$\par
{-} \gk{\$} B = (28,  0) \$ \gc{since} \$ AB =  28 \$\par
{-} \gk{\$} C = (0, \gc{h}) \$ \gc{for} some \gk{height} \$ h \gc{\$}\par
\vspace{0.3\baselineskip}
\gc{Now} \gc{place} \gk{points} \gk{\$} D, E \$ on \$ AB \$\gk{,} \gk{and} \$ F, G \$ on \$ AC \$ \gc{accordingly}.\par
\vspace{0.3\baselineskip}
\gc{**}\gc{Points} on \$ AB \$:**\par
\vspace{0.3\baselineskip}
{-} \$ AD =  4 \textbackslash{}Rightarrow D = (4,  0) \$\par
{-} \$ \gc{AE} = AD + DE =  4 +  16 =  20 \textbackslash{}Rightarrow E = (20,  0) \$\par
{-} \$ \gk{EB} =  8 \textbackslash{}Rightarrow \gk{B} = (28,  0) \$, \gc{which} \gc{checks} \gk{out}\par
\vspace{0.3\baselineskip}
**Points on \$ AC \$:**\par
\vspace{0.3\baselineskip}
\gc{We} \gc{will} \gc{find} \gc{the} \gk{coordinates} of \$ \gk{F} \$ and \$ G \$\gc{,} \gc{then} \gk{use} \gk{the} \gc{area} \gc{formula} \gc{to} \gc{compute} \gk{the} \gk{area} of \gc{\$} \gc{\textbackslash{}}triangle ABC \$\gc{.}\par
\vspace{0.3\baselineskip}
\gc{Since} \$ \gc{C} = (0, h) \$, \gc{then} \gk{the} \gc{vector} \gk{from} \$ A \$ to \$ C \$ is \$ \textbackslash{}vec\{AC\} = (0, h) \gk{\$}\par
\par}
\end{sraexbox}
\end{figure*}

\begin{figure*}[p]
\begin{sraexbox}{Example~\thesraexample: AIME 2025 problem under SRA (2/3)}
{\ttfamily\small\raggedright\setlength{\parindent}{0pt}\setlength{\parskip}{0pt}
\gc{{-}} \$ AF =  13 \gk{\$}, \gk{so} \gc{the} \gc{vector} \gc{representing} \gk{\$} \gc{F} \$ \gk{is} \gk{\$} \gk{\textbackslash{}}\gk{frac}\{13\}\{91\} \textbackslash{}vec\{AC\} = \textbackslash{}left(0, \textbackslash{}frac\{13h\}\{91\gk{\}\textbackslash{}}right) \gc{=} \textbackslash{}left(0, \textbackslash{}frac\{h\}\{7\}\textbackslash{}right) \gk{\$}\par
~ ~So \gc{\$} F = \textbackslash{}left(0, \textbackslash{}frac\{h\}\{7\}\textbackslash{}right) \$\par
\vspace{0.3\baselineskip}
{-} \$ \gk{FG} =  52 \$, \gk{so} \gc{\$} \gk{G} \gc{=} \gk{F} + \textbackslash{}frac\{52\}\{91\} \textbackslash{}vec\{AC\} = \textbackslash{}left(0, \textbackslash{}frac\{h\}\{7\} + \textbackslash{}frac\{52h\}\{91\}\textbackslash{}right) \gk{=} \textbackslash{}left(0, \textbackslash{}frac\{\gk{1}3h \gc{+}  52h\}\{91\} \textbackslash{}right) = \textbackslash{}left(0, \textbackslash{}frac\{65h\}\{91\} \textbackslash{}right) \gk{=} \textbackslash{}left(0, \textbackslash{}frac\{5h\}\{7\} \textbackslash{}right) \$\par
\vspace{0.3\baselineskip}
\gk{So}\gc{,} \gk{the} \gk{coordinates} are:\par
\vspace{0.3\baselineskip}
{-} \$ A = (0,  0) \$\par
{-} \$ \gk{B} = (28,  0) \$\par
{-} \$ C = (0, h) \$\par
{-} \$ D = (4,  0) \$\par
{-} \$ E = (20,  0) \$\par
{-} \$ F = \textbackslash{}left(0, \textbackslash{}frac\{h\}\{7\} \textbackslash{}right) \$\par
{-} \$ G = \textbackslash{}left(0, \textbackslash{}frac\{5h\}\{7\} \textbackslash{}right) \$\par
\vspace{0.3\baselineskip}
\gc{Now}\gk{,} \gk{we} \gc{compute} \gk{the} \gk{area} of triangle \$ \gk{ABC} \$\gc{,} \gk{and} \gk{use} \gk{the} \gk{area} of quadrilateral \$ DEGF \$ to \gk{find} \$ h \$.\par
\vspace{0.3\baselineskip}
\gc{The} area of triangle \$ ABC \$ is:\par
\vspace{0.3\baselineskip}
\$\$\par
\textbackslash{}text\{Area\gk{\}\_\{}ABC\} = \textbackslash{}frac\{1\}\{2\} \textbackslash{}times \gk{AB} \textbackslash{}times \gk{h} = \textbackslash{}frac\{1\}\{2\} \textbackslash{}times  28 \textbackslash{}times h =  14h\par
\$\$\par
\vspace{0.3\baselineskip}
\gc{Now}\gk{,} \gc{the} \gk{area} of quadrilateral \$ DEGF \$ is \gc{the} \gc{sum} of \gk{the} areas of \gk{triangles} \$ \gc{D}\gk{FG} \$ and \$ \gc{E}\gk{FG} \$\gk{,} \gc{but} \gc{a} \gk{better} \gc{way} is to \gk{use} \gk{the} \gk{coordinates} \gk{of} \gk{the} \gc{four} points \gc{and} \gc{use} the \gc{**}\gk{sh}oelace formula**.\par
\vspace{0.3\baselineskip}
\gk{Let}\textquotesingle{}s \gc{list} \gc{them} \gk{in} order\gk{:}\par
\vspace{0.3\baselineskip}
\gk{{-}} \$ D = (4,  0) \$\par
{-} \$ \gk{E} = (20,  0) \$\par
{-} \$ G = (0, \textbackslash{}frac\{5h\}\{7\}) \$\par
{-} \$ F = (0, \textbackslash{}frac\{h\}\{7\}) \gk{\$}\par
\vspace{0.3\baselineskip}
\gc{Use} the shoelace formula\gk{:}\par
\vspace{0.3\baselineskip}
\$\$\par
\textbackslash{}text\{Area\}\_\{DEGF\} = \textbackslash{}frac\{1\}\{2\} \textbackslash{}left| \gk{x}\_1y\_2 + x\_2y\_3 + x\_3y\_4 + x\_4y\_1 {-} (y\_1x\_2 + y\_2x\_3 + y\_3x\_4 + y\_4x\_1) \textbackslash{}right|\par
\$\$\par
\vspace{0.3\baselineskip}
\gc{Sub}stitute \gc{the} points:\par
\vspace{0.3\baselineskip}
\gk{{-}} \$ D = (4,  0) \$\par
{-} \$ E = (20,  0) \$\par
{-} \$ G = (0, \textbackslash{}frac\{5h\}\{7\}) \$\par
{-} \$ F = (0, \textbackslash{}frac\{h\}\{7\}) \$\par
\vspace{0.3\baselineskip}
\gc{Sh}oelace \gc{formula}\gk{:}\par
\vspace{0.3\baselineskip}
\$\$\par
\textbackslash{}text\{Area\}\_\{DEGF\} = \textbackslash{}frac\{1\}\{2\} \textbackslash{}left|  4 \textbackslash{}cdot  0 +  20 \textbackslash{}cdot \textbackslash{}frac\{5h\}\{7\} +  0 \textbackslash{}cdot \textbackslash{}frac\{h\}\{7\} +  0 \textbackslash{}cdot  0 {-} (0 \textbackslash{}cdot  20 +  0 \textbackslash{}cdot  0 + \textbackslash{}frac\{5h\}\{7\} \textbackslash{}cdot  0 + \textbackslash{}frac\{h\}\{7\} \textbackslash{}cdot  4) \textbackslash{}right|\par
\$\$\par
\vspace{0.3\baselineskip}
\gk{\$\$}\par
= \textbackslash{}frac\{1\}\{2\} \textbackslash{}left|  0 + \textbackslash{}frac\{100h\}\{7\} +  0 +  0 {-} \gk{(}0 +  0 +  0 + \textbackslash{}frac\{4h\}\{7\}) \textbackslash{}right| = \textbackslash{}frac\{1\}\{2\} \textbackslash{}left| \textbackslash{}frac\{100h\}\{7\} {-} \textbackslash{}frac\{4h\}\{7\} \textbackslash{}right| = \textbackslash{}frac\{1\}\{2\} \textbackslash{}cdot \textbackslash{}frac\{96h\}\{7\} = \textbackslash{}frac\{48h\}\{7\}\par
\$\$\par
\vspace{0.3\baselineskip}
We are told \gc{this} \gk{area} is  288:\par
\vspace{0.3\baselineskip}
\$\$\par
\textbackslash{}frac\{48h\}\{7\} =  288 \textbackslash{}Rightarrow h = \textbackslash{}frac\{288 \textbackslash{}times  7\}\{48\} = \textbackslash{}frac\{2016\}\{48\} =  42\par
\$\$\par
\vspace{0.3\baselineskip}
So\gk{,} \gc{the} \gk{height} \gk{\$} h =  42 \$\par
\vspace{0.3\baselineskip}
\gk{{-}{-}{-}}\par
\vspace{0.3\baselineskip}
\#\#\# Step  4: \gc{Use} \gc{the} \gk{coordinates} \gc{with} \$ h =  42 \$\par
\par}
\end{sraexbox}
\end{figure*}

\begin{figure*}[p]
\begin{sraexbox}{Example~\thesraexample: AIME 2025 problem under SRA (3/3)}
{\ttfamily\small\raggedright\setlength{\parindent}{0pt}\setlength{\parskip}{0pt}
\gc{{-}} \$ \gk{A} = (0,  0) \$\par
{-} \$ B = (28,  0) \$\par
{-} \$ C = (0,  42) \$\par
{-} \$ D = (4,  0) \$\par
{-} \$ E = (20,  0) \$\par
{-} \$ F = \textbackslash{}left(0, \textbackslash{}frac\{42\}\{7\} \textbackslash{}right) = (0,  6) \$\par
{-} \$ G = \textbackslash{}left(0, \textbackslash{}frac\{\gc{2}10\}\{7\} \textbackslash{}right) = (0,  30) \$\par
\vspace{0.3\baselineskip}
\gk{Now} \gc{compute} \gk{the} \gk{coordinates} of \gk{\$} M \$ and \$ N \$\gk{:}\par
\vspace{0.3\baselineskip}
{-} \$ M \$ is the reflection of \$ D \gk{\$} through \$ F \gk{\$}\gk{,} \gk{so} \gc{since} \$ \gk{F} \$ is \gk{the} midpoint of \$ DM \$, \gk{then}\gk{:}\par
\vspace{0.3\baselineskip}
\$\$\par
\gc{\textbackslash{}}\gk{vec}\{\gk{FM}\} = \textbackslash{}vec\{FD\gk{\}} \textbackslash{}Rightarrow \gk{M} =\gk{~}2\gk{F} {-} D =  2(0,  6) {-} (4,  0) = ({-}4,  12)\par
\$\$\par
\vspace{0.3\baselineskip}
{-} \$ N \$ is the reflection of \$ G \$ through \$ E \$, so \gc{since} \$ E \$ is the midpoint of \$ GN \$, then:\par
\vspace{0.3\baselineskip}
\$\$\par
\textbackslash{}vec\{EN\} = \textbackslash{}vec\{EG\} \textbackslash{}Rightarrow N =  2E {-} G =  2(20,  0) {-} (0,  30) = (40, {-}30)\par
\$\$\par
\vspace{0.3\baselineskip}
\gk{Now}\gc{,} \gc{list} all \gc{vertices} of heptagon \$ AFNBCEM \$ \gk{in} order\gk{:}\par
\vspace{0.3\baselineskip}
\gc{\$\$}\par
\gk{A} = (0,  0), \gc{\textbackslash{}}quad F = (0,  6), \textbackslash{}quad N = (40, {-}30), \textbackslash{}quad B = (28,  0), \textbackslash{}quad C = (0,  42), \textbackslash{}quad E = (20,  0), \textbackslash{}quad M = ({-}4,  12)\par
\$\$\par
\vspace{0.3\baselineskip}
\gc{Use} the **shoelace formula** \gc{to} compute the area of \gc{this} heptagon.\par
\vspace{0.3\baselineskip}
\gk{List} \gk{the} \gk{points} \gk{in} order \gc{and} \gk{repeat} the \gk{first} \gk{point}\gc{:}\par
\vspace{0.3\baselineskip}
\gk{\$\$}\par
\gk{(}0,  0), (0,  6), (40, {-}30), (28,  0), (0,  42), (20,  0), ({-}4,  12), (0,  0)\par
\$\$\par
\vspace{0.3\baselineskip}
\gc{Now} \gk{apply} the shoelace formula:\par
\vspace{0.3\baselineskip}
\$\$\par
\textbackslash{}text\{Area\} = \textbackslash{}frac\{1\}\{2\} \textbackslash{}left| \gk{\textbackslash{}}sum x\_i y\_\{i+1\} {-} \textbackslash{}sum y\_i x\_\{i+1\} \textbackslash{}right|\par
\$\$\par
\vspace{0.3\baselineskip}
\gk{Compute} \gc{first} \gk{sum}\gc{:} \gk{\$} \textbackslash{}sum x\_i y\_\{i+1\} \$:\par
\vspace{0.3\baselineskip}
{-} \$  0 \textbackslash{}cdot  6 \gk{=}  0 \$\par
{-} \$  0 \textbackslash{}cdot {-}30 =  0 \$\par
{-} \$  40 \textbackslash{}cdot  0 =  0 \$\par
{-} \$  28 \textbackslash{}cdot  42 =  1176 \$\par
{-} \$  0 \textbackslash{}cdot  0 =  0 \$\par
{-} \$  20 \textbackslash{}cdot  12 =  240 \$\par
{-} \$ {-}4 \textbackslash{}cdot  0 =  0 \gc{\$}\par
\vspace{0.3\baselineskip}
\gc{So} \gc{sum} \gk{=} \$  0 +  0 +  0 +  1176 +  0 +  240 +  0 =  1416 \$\par
\vspace{0.3\baselineskip}
\gc{Compute} second sum: \$ \textbackslash{}sum y\_i x\_\{i+1\} \$\par
\vspace{0.3\baselineskip}
{-} \$  0 \textbackslash{}cdot  0 =  0 \$\par
{-} \$  6 \textbackslash{}cdot  40 =  240 \$\par
{-} \$ {-}30 \textbackslash{}cdot  28 = {-}840 \$\par
{-} \$  0 \textbackslash{}cdot  0 =  0 \$\par
{-} \$  42 \textbackslash{}cdot  20 =  840 \$\par
{-} \$  0 \textbackslash{}cdot {-}4 =  0 \$\par
{-} \$  12 \textbackslash{}cdot  0 =  0 \$\par
\vspace{0.3\baselineskip}
So sum = \$  0 +  240 {-}  840 +  0 +  840 +  0 +  0 =  240 \$\par
\vspace{0.3\baselineskip}
\gk{Now}\gc{:}\par
\vspace{0.3\baselineskip}
\$\$\par
\textbackslash{}text\{Area\} = \textbackslash{}frac\{1\}\{2\} \textbackslash{}left|  1416 {-}  240 \textbackslash{}right| = \textbackslash{}frac\{1\}\{2\} \textbackslash{}cdot  1176 =  588\par
\$\$\par
\vspace{0.3\baselineskip}
{-}{-}{-}\par
\vspace{0.3\baselineskip}
\#\#\# \gk{Final} Answer:\par
\vspace{0.3\baselineskip}
\$\$\par
\textbackslash{}boxed\{588\}\par
\$\$\par
\par}
\end{sraexbox}
\end{figure*}
\clearpage

\end{document}